\documentclass[letterpaper,twocolumn,10pt]{article}
\usepackage[left=0.625in,right=0.625in,top=0.75in,bottom=0.625in,columnsep=0.18in,footskip=20pt]{geometry}
\usepackage[T1]{fontenc}
\usepackage{newtxtext}
\usepackage{microtype}
\usepackage{titlesec}
\usepackage{amsmath,amssymb,amsthm,bm}
\usepackage{graphicx,xcolor}
\usepackage{caption}
\usepackage{algorithm,algpseudocodex}
\usepackage{tabularray,catchfile}
\UseTblrLibrary{booktabs}
\usepackage[backend=biber,style=ieee,sorting=none,isbn=false]{biblatex}
\usepackage{xurl}
\DeclareFieldFormat{doi}{%
  \mkbibacro{DOI}\addcolon\space
  {\urlstyle{same}%
   \ifhyperref{\href{https://doi.org/#1}{\nolinkurl{#1}}}{\nolinkurl{#1}}}%
  \nopunct}
\definecolor{linkblue}{RGB}{0,0,255}
\usepackage[a-2b]{pdfx}
\hypersetup{colorlinks=true,linkcolor=linkblue,citecolor=linkblue,urlcolor=black,hypertexnames=false}
\usepackage{orcidlink}
\usepackage[capitalise,noabbrev]{cleveref}
\usepackage{xspace}
\usepackage{flushend}
\titlespacing*{\section}{0pt}{2.5ex plus 0.6ex minus 1ex}{1.2ex plus 0.3ex minus 0.4ex}
\titlespacing*{\subsection}{0pt}{2ex plus 0.5ex minus 0.8ex}{0.8ex plus 0.2ex minus 0.3ex}
\titlespacing*{\subsubsection}{0pt}{1.6ex plus 0.4ex minus 0.6ex}{0.6ex plus 0.2ex minus 0.2ex}
\newtheorem{theorem}{Theorem}
\renewcommand{\Vec}[1]{\bm{#1}}
\newcommand{\Mat}[1]{\bm{#1}}
\newcommand{\inv}{{-1}}
\newcommand{\T}{\top}
\newcommand{\FN}[1]{{\lVert{#1}\rVert_\mathrm{F}}}
\newcommand{\SFN}[1]{{\lVert{#1}\rVert_\mathrm{F}^2}}
\newcommand{\dataset}[1]{\mathcal{#1}}
\newcommand{\argmin}[1]{\underset{#1}{\operatorname{argmin}}~}
\newcommand{\onehot}{\operatorname{onehot}}

\newcommand{\ve}[2]{\mbox{#1{\tiny$\pm$#2}}}
\newcommand{\Ve}[2]{\underline{\ve{#1}{#2}}}
\newcommand{\VE}[2]{\textbf{\ve{#1}{#2}}}
\newcommand{\aauc}{$\mathcal{A}_{\text{auc}}$}
\newcommand{\aavg}{$\mathcal{A}_{\text{avg}}$}
\newcommand{\alst}{$\mathcal{A}_{\text{last}}$}
\newcommand{\ascore}{$\mathcal{A}$}
\newcommand{\maauc}{\mathcal{A}_{\text{auc}}}
\newcommand{\maavg}{\mathcal{A}_{\text{avg}}}
\newcommand{\malst}{\mathcal{A}_{\text{last}}}
\newcommand{\fauc}{$\mathcal{F}_{\text{auc}}$}
\newcommand{\favg}{$\mathcal{F}_{\text{avg}}$}
\newcommand{\flst}{$\mathcal{F}_{\text{last}}$}
\newcommand{\fscore}{$\mathcal{F}$}
\newcommand{\mfauc}{\mathcal{F}_{\text{auc}}}
\newcommand{\mfavg}{\mathcal{F}_{\text{avg}}}
\newcommand{\mflst}{\mathcal{F}_{\text{last}}}
\newcommand{\eg}{e.g.\@\xspace}
\newcommand{\ie}{i.e.\@\xspace}

\title{AIR: Analytic Imbalance Rectifier for Continual Learning}
\author{
  Di Fang\textsuperscript{1}\,\orcidlink{0009-0004-8135-2354},
  Yinan Zhu\textsuperscript{1}\,\orcidlink{0009-0006-8730-6695},
  Zhiping Lin\textsuperscript{2}\,\orcidlink{0000-0002-1587-1226},
  Cen Chen\textsuperscript{1}\,\orcidlink{0000-0003-1389-0148},
  Ziqian Zeng\textsuperscript{1}\,\orcidlink{0000-0003-0060-7956},
  Huiping Zhuang\textsuperscript{1}\,\orcidlink{0000-0002-4612-5445}\thanks{Corresponding author: Huiping Zhuang (hpzhuang@scut.edu.cn)}\\[0.5em]
  \textsuperscript{1}South China University of Technology, \hspace{0.25cm}\textsuperscript{2}Nanyang Technological University
}
\date{}

\begin{document}
\maketitle
\begin{abstract}
    Continual learning (CL) agents incrementally learn from sequentially arriving data and adapt to the dynamic, ever-changing nature of real-world environments.
    However, many existing CL methods suffer performance degradation in evolving, imbalanced data streams due to limited adaptation to changing class frequencies or ineffective use of mixed data from new and previously observed classes.
    To deal with these challenges, we propose an analytic imbalance rectifier (AIR) algorithm for real-world CL.
    AIR is an online exemplar-free approach with a frozen backbone as the feature extractor and a closed-form incremental classifier whose weight equals the joint-learning weight for the same class-weighted ridge objective.
    AIR addresses class imbalance with an analytic reweighting module (ARM) that calculates a reweighting factor for each class in the loss function to equalize total sample weights across classes.
    Under long-tailed class-incremental learning, AIR leads 28 baselines in aggregate accuracy and exemplar-free methods in aggregate macro F1, gaining 3.21\% accuracy and 2.14\% macro F1 over the respective strongest exemplar-free baselines.
    Under the Si-Blurry setting with recurring classes, AIR leads 15 exemplar-based and exemplar-free baselines, gaining 2.32\% aggregate accuracy and 1.27\% aggregate macro F1 over the strongest baseline.
    One-sided paired tests support positive mean absolute gains in these four comparisons (Holm-adjusted $p<0.006$).
\end{abstract}


\section{Introduction}\label{sec:intro}
\par Continual learning (CL) is a machine learning paradigm that imitates humans' life-long learning ability to adapt to the dynamic, ever-changing nature of real-world environments. CL enables AI models to be trained on new data sequentially and improve their abilities. Exploring this paradigm is essential for reducing the considerable cost of retraining deep neural networks, especially for large pre-trained models. Many methods focus on class-incremental learning (CIL), one of the most challenging paradigms in CL, aiming to address the severe catastrophic forgetting problem \cite{CF_Bower_PLM1989} where models tend to forget previously learned knowledge.

\par Many conventional CIL benchmarks assume phase-disjoint, balanced classes, whereas real-world data streams may be imbalanced and contain recurring classes (a.k.a blurry classes). These real-world scenarios can be characterized by long-tailed CIL (LT-CIL) \cite{LTCIL_Liu_ECCV2022} and generalized CIL (GCIL) \cite{BlurryM_Aljundi_NeurIPS2019} in \Cref{fig:imbalanced_settings}. LT-CIL refers to the CIL scenarios with a long-tailed class frequency distribution, highlighting the real-world class imbalance problem. GCIL refers to the CL scenarios where new and old classes are mixed in the same phase. It focuses on the dynamic sample number for each class and blurry class distributions, such as Si-Blurry \cite{MVP_Moon_ICCV2023}.

\begin{figure}[ht]
    \centering
    \includegraphics[width=\linewidth]{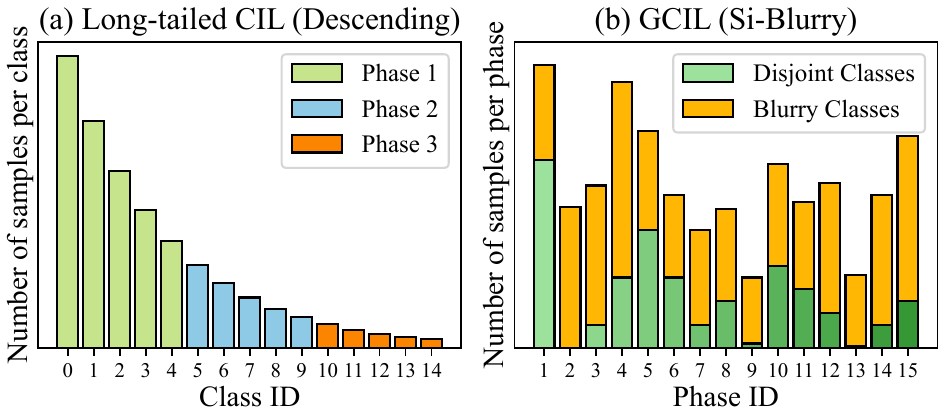}
    \caption{Different settings of real-world class-imbalanced CIL.}\label{fig:imbalanced_settings}
\end{figure}

\par Existing CL methods face a significant performance decline under real-world scenarios due to
(1) \textit{class imbalance}: unequal training sample counts across classes can bias predictions towards instance-rich (head) classes at the expense of instance-rare (tail) classes;
(2) \textit{data mixture}: models need to learn real-world blurry data in a streaming fashion, but the sample number distribution changes dynamically, leading to intra- and inter-phase forgetting and the dynamic imbalance problem \cite{MVP_Moon_ICCV2023};
and (3) \textit{data privacy}: real-world applications without original data storage permissions forbid methods that store past training samples as exemplars to avoid catastrophic forgetting.

\par The aforementioned real-world properties pose significant challenges for building practical deep CL systems. On the one hand, the skewed class weighting of unweighted squared-error loss under class-imbalanced CL scenarios can lead to substantial performance degradation. For instance, recent analytic CL (ACL) techniques \cite{ACIL_Zhuang_NeurIPS2022} provide recursive closed-form solutions, recovering the joint ridge classifier for fixed features but also inheriting this frequency-proportional weighting. On the other hand, existing deep long-tailed learning methods \cite{Survey:LTL_Zhang_2023} largely focus on static datasets and require adaptation to shifting imbalance. However, label frequencies change during CL, and maintaining the current class-weighted joint solution requires reweighting previously observed data without retaining exemplars.

\par To address these challenges, we propose an analytic imbalance rectifier (AIR), an online exemplar-free approach with a closed-form solution for class-imbalanced CIL. AIR introduces an analytic reweighting module (ARM) that calculates reweighting factors to equalize total sample weights across classes. AIR provides a closed-form incremental solution that provably recovers the joint solution of the same class-weighted ridge objective for fixed features. Key contributions of this paper include:

\begin{itemize}
    \item We propose AIR, an online exemplar-free CL method that addresses dynamic class imbalance in both phase-disjoint and recurring-class streams.
    \item We reveal the class-level imbalance hidden by uniform sample weighting: unweighted analytic classifiers assign larger coefficients to head-class mean losses, which provides the theoretical basis for AIR's imbalance rectification.
    \item We develop ARM to equalize total sample weights across classes without exemplar replay, yielding a closed-form incremental solution provably equivalent to joint fitting of the same fixed-feature weighted ridge objective.
    \item AIR achieves significant accuracy and macro F1 gains over the strongest exemplar-free LT-CIL baselines and the strongest Si-Blurry baseline, while ARM also improves existing ACL methods, demonstrating broader applicability.
\end{itemize}

\section{Related Works}

\subsection{Conventional CIL for Class-Balanced Data}
    \par Conventional CIL usually assumes phase-disjoint, approximately balanced classes. Replay-based methods retain old samples: ER \cite{ER_Rolnick_NeurIPS2019} is a basic replay baseline, while LUCIR \cite{LUCIR_Hou_CVPR2019} and PODNet \cite{PODNet_Douillard_ECCV2020} combine replay with distillation. Exemplar-free baselines instead constrain model updates. EWC \cite{EWC_Kirkpatrick_PNAS2017} regularizes important parameters, and LwF \cite{LwF_Li_TPAMI2018} distills predictions from the previous model.

    \par Recent advances have also motivated CIL on pre-trained ViTs \cite{ViT_Dosovitskiy_ICLR2021}: DualPrompt \cite{DualPrompt_Wang_ECCV2022}, CODA-Prompt \cite{CODA-Prompt_Smith_CVPR2023}, MVP \cite{MVP_Moon_ICCV2023}, and MISA \cite{MISA_Kang_ICLR2025} use prompt-based adaptation, while SinglePrompt \cite{SinglePrompt_Park_CVPRF2026} uses a single prompt for task-free online learning. LAE \cite{LAE_Gao_ICCV2023} supports general parameter-efficient modules, and Online-LoRA \cite{Online-LoRA_Wei_WACV2025} performs task-free low-rank adaptation. SimpleCIL and APER \cite{SimpleCIL_Zhou_IJCV2025}, FeCAM \cite{FeCAM_Goswami_NeurIPS2023}, and MOS \cite{MOS_Sun_AAAI2025} exploit pre-trained representations with incremental classifiers or model merging.

    \par Analytic CL decouples a frozen feature extractor from a closed-form classifier: ACIL \cite{ACIL_Zhuang_NeurIPS2022} updates the classifier recursively, while RanPAC \cite{RanPAC_McDonnell_NeurIPS2023} accumulates covariance statistics. DS-AL \cite{DS-AL_Zhuang_AAAI2024} maintains two analytic classifiers to enhance classification performance, and GACL \cite{GACL_Zhuang_NeurIPS2024} extends analytic CL to GCIL. These ridge objectives assign unit weight to each sample and can favor frequent classes in imbalanced streams. SLDA \cite{SLDA_Hayes_CVPR2020W} is a related streaming classifier based on linear discriminant analysis, rather than ridge regression. AIR belongs to analytic CL, but rectifies frequency-proportional class weighting with ARM, thereby addressing the class-imbalance problem.

\subsection{Long-Tailed CIL (LT-CIL)}
    \par To address the class imbalance problem in deep long-tailed CIL, several approaches are proposed, including LUCIR \cite{LUCIR_Hou_CVPR2019}, BiC \cite{BiC_Wu_CVPR2019}, and PRS \cite{PRS_Kim_ECCV2020}. \textcite{LTCIL_Liu_ECCV2022} propose a two-stage learning paradigm, bridging existing CIL methods to LT-CIL. Subsequent works include DRC \cite{DRC_Chen_ICCV2023} for dynamic residual classification, DGR \cite{DGR_He_CVPR2024} for gradient reweighting, and DAP \cite{DAP_Hong_IJCAI2024} for balancing stability and plasticity through boosting and stabilizing prompt anchors. APART \cite{APART_Qi_ML2025} uses adaptive adapter routing, and PSR \cite{PSR_Fu_PR2026} regulates the feature space for tail-class representation. For online learning, AEF-OCL \cite{AEF-OCL_Zhuang_TVT2025} uses an analytic exemplar-free classifier but inevitably introduces noise when it tries to balance the class distribution with synthesized prototypes. MLG \cite{MLG_Liang_PR2025} introduces batch-level logit mask and batch-level feature cross fusion for class-imbalance, and accumulative mean feature distillation to alleviate abrupt feature drift.

\subsection{Generalized CIL (GCIL)}
    \par Several GCIL settings are proposed to simulate real-world CL with overlapping classes and changing sample counts. In the BlurryM \cite{BlurryM_Aljundi_NeurIPS2019} setting, classes recur across phases with different dominant classes. The i-Blurry-N-M \cite{CLIB_Koh_ICLR2022} setting has blurry phase boundaries and requires the model to perform inference at any time. It has a fixed class number in each phase with the same proportion of new and old classes, while the Si-Blurry \cite{MVP_Moon_ICCV2023} setting has an ever-changing class number, giving a more realistic simulation of real-world CL. RM \cite{RM_Bang_CVPR2021}, CLIB \cite{CLIB_Koh_ICLR2022}, DualPrompt \cite{DualPrompt_Wang_ECCV2022}, MVP \cite{MVP_Moon_ICCV2023}, UniCIL \cite{UniCIL_Yang_arXiv2023}, and GACL \cite{GACL_Zhuang_NeurIPS2024} are evaluated under GCIL.

\section{Method}
    \par We formulate the CIL problem in \Cref{sec:CIL}. \Cref{sec:features} shows how AIR extracts features. \Cref{sec:ACL} discusses the ridge regression classifier without rectification, and \Cref{sec:loss} reveals the class-level imbalance induced by uniform sample weighting. \Cref{sec:ARM} introduces the key idea of AIR, and \Cref{section:G-AIR} extends AIR to GCIL.

    \begin{figure*}[t]
        \centering
        \includegraphics[width=\linewidth]{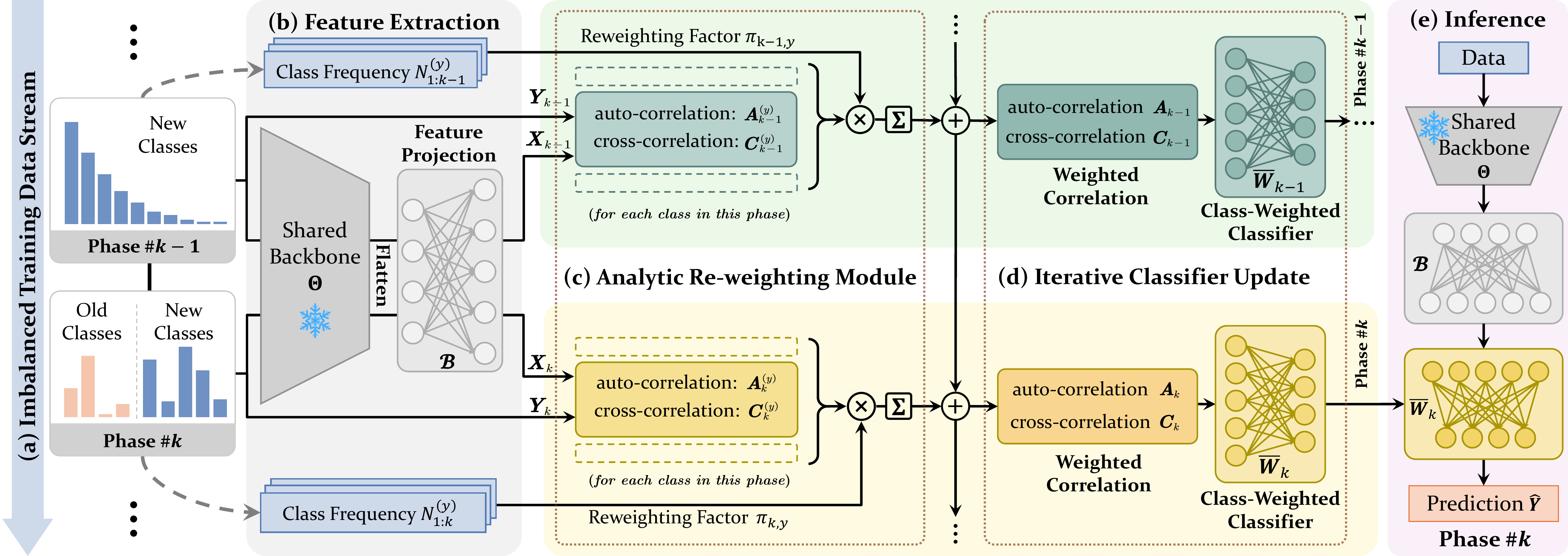}
        \caption{
            AIR includes
            (\textbf{a}) the input data stream that arrives phase by phase with dynamically changing class numbers, where data may be imbalanced (\eg, phase $k-1$) or in blurry order (\eg, phase $k$);
            (\textbf{b}) a frozen backbone followed by a random feature projection layer that extracts features mapped into a higher dimensional space;
            (\textbf{c}) the analytic reweighting module (ARM) calculating the reweighting factor $\pi_{k,y}$ for each class;
            (\textbf{d}) the class-weighted classifier updated iteratively at each phase;
            (\textbf{e}) the backbone, the projection layer, and the classifier combined for inference.
        }\label{fig:flowchart}
    \end{figure*}

    \subsection{Class-Incremental Learning Problem}\label{sec:CIL}
        \par Let $\{\dataset{D}_1, \dataset{D}_2, \dots\}$ be the classification dataset arriving phase by phase, as shown in \Cref{fig:flowchart} (a). The phase-$k$ training set is $\dataset{D}_k = \{(\Mat{\mathcal{X}}_{k,1}, y_{k,1}),\allowbreak \dots,\allowbreak (\Mat{\mathcal{X}}_{k,N_k}, y_{k,N_k})\}$, where $N_k$ is its size, $\Mat{\mathcal{X}}$ is an input tensor, and $y$ is an integer class label. Let $C_k$ be the total number of classes observed through phase $k$. Without loss of generality, their labels are $\{0,\dots,C_k-1\}$.
        \par In conventional CIL, the classes introduced in different phases are strictly disjoint and $C_k<C_{k+1}$. In GCIL, a class in the current phase may have appeared previously, and $C_k\leq C_{k+1}$.

    \subsection{Fixed Feature Extraction and Random Projection}\label{sec:features}
        \par AIR extracts features with a frozen backbone network followed by a frozen random feature projection, shown in \Cref{fig:flowchart} (b). The backbone parameters $\Mat{\Theta}$ are trained on the base dataset $\dataset{D}_1$ or pre-trained on a large-scale dataset. The projection layer $\mathcal{B}$ non-linearly projects features to a higher-dimensional space \cite{Geometrical_Thomas_TEC1965}. We write each extracted feature as a row vector $\Vec{x}$:
        \begin{equation}\label{eq:buffer}
            \Vec{x} = \mathcal{B}(f_\text{backbone}(\Mat{\mathcal{X}}, \Mat{\Theta})).
        \end{equation}
        \noindent Several projection layer designs are available. We follow ACIL \cite{ACIL_Zhuang_NeurIPS2022} and GACL \cite{GACL_Zhuang_NeurIPS2024} to use a frozen, randomly initialized projection matrix $\Mat{W}_{\text{rand}}$ followed by a ReLU activation function, \ie, $\mathcal{B}(\Vec{x}) = \operatorname{ReLU}(\Vec{x}\Mat{W}_{\text{rand}})$.
        \par AIR decouples the feature extractor from the classifier. The entire feature map remains fixed throughout analytic updates. Let $f$ denote its output (projection) dimension. The classifier maps an extracted feature to a vector of class scores. Let $\Mat{X}_k$ and $\Mat{Y}_k$ stack the extracted features and their one-hot labels in $\dataset{D}_k$, respectively. Stacking these matrices over phases $1$ through $k$ gives $\Mat{X}_{1:k}$ and $\Mat{Y}_{1:k}$. Label and cross-correlation matrices are zero-padded with columns for newly observed classes.

    \subsection{Ridge Regression Without Rectification}\label{sec:ACL}
        \par Before introducing AIR, we consider the unweighted ridge objective used by existing ACL methods at phase $k$:
        \begin{equation}\label{eq:ACL_loss}
            \mathcal{L}(\Mat{W}_{k}) = \SFN{\Mat{X}_{1:k}\Mat{W}_{k} - \Mat{Y}_{1:k}} + \gamma \SFN{\Mat{W}_{k}},
        \end{equation}
        where $\FN{\cdot}$ indicates the Frobenius norm and $\gamma$ is the regularization term coefficient.
        \par With this loss, existing ACL methods find a recursive form \cite{ACIL_Zhuang_NeurIPS2022} or an iterative form \cite{RanPAC_McDonnell_NeurIPS2023} of the optimal solution at phase $k$, thereby achieving online CL.
        \begin{equation}\label{eq:ACL_W}
            \Mat{\hat{W}}_{k} = \argmin{\Mat{W}_{k}}\mathcal{L}(\Mat{W}_{k}) = (\sum_{t=1}^{k}\Mat{A}_{t} + \gamma \Mat{I})^{\inv}(\sum_{t=1}^{k}\Mat{C}_{t}),
        \end{equation}
        where $\Mat{A}_{t} = \Mat{X}_{t}^{\T}\Mat{X}_{t}$ is the \textit{auto-correlation feature matrix}, and $\Mat{C}_{t} = \Mat{X}_{t}^{\T}\Mat{Y}_{t}$ is the \textit{cross-correlation feature matrix}.
        \par However, the classifier trained with the loss \eqref{eq:ACL_loss} without rectification is biased and faces severe performance degradation under class-imbalanced scenarios.

    \subsection{Class Imbalance in Unweighted Ridge Fitting}\label{sec:loss}
        \par The loss in \cref{eq:ACL_loss} assigns each sample unit weight, so a class's total sample weight is proportional to its sample count.

        \par We sort the samples at each phase by their labels to illustrate this issue. Let $\Vec{x}_{k,i}^{(y)}$ be the $i$-th extracted feature vector with label $y$ at phase $k$. Similarly, we use $\Mat{X}_{k}^{(y)}$ and $\Mat{Y}_{k}^{(y)}$ to represent the extracted features and labels with the same label $y$ at phase $k$. $\Mat{X}_{1:k}^{(y)}$ and $\Mat{Y}_{1:k}^{(y)}$ are all the features and labels with the same label $y$ from phase $1$ to $k$. $N_{k}^{(y)}$ is the sample number at phase $k$ with label $y$, and $N_{1:k}^{(y)} = \sum_{t=1}^{k}N_{t}^{(y)}$ is the number of all training samples with label $y$.

        \par Rearranging the samples by their labels, the training loss \eqref{eq:ACL_loss} can be written as the sum of the \textit{class-specific loss} $\mathcal{L}^{(y)}(\Mat{W}_{k})$ for each class $y$ plus the regularization term:

        \begin{align}\label{eq:ACL_loss_by_class}
        \mathcal{L}(\Mat{W}_{k})
            &= \sum_{t=1}^{k}\sum_{i=1}^{N_t}\SFN{\Vec{x}_{t,i}\Mat{W}_{k} - \onehot(y_{t,i})} + \gamma\SFN{\Mat{W}_{k}} \notag \\
            &= \sum_{y=0}^{C_k-1}\mathcal{L}^{(y)}(\Mat{W}_{k}) + \gamma\SFN{\Mat{W}_{k}},
        \end{align}
        where the \textit{class-specific loss} $\mathcal{L}^{(y)}(\Mat{W}_{k})$ is
        \begin{equation}
            \begin{aligned}\label{eq:class_specific_loss}
                \mathcal{L}^{(y)}(\Mat{W}_{k})
                    &= \sum_{t=1}^{k}\sum_{i=1}^{N_{t}^{(y)}}\SFN{\Vec{x}_{t,i}^{(y)}\Mat{W}_{k} - \onehot(y)} \\
                    &= \SFN{\Mat{X}_{1:k}^{(y)}\Mat{W}_{k} - \Mat{Y}_{1:k}^{(y)}}.
            \end{aligned}
        \end{equation}

        \par Each training sample receives equal weight in the total loss $\mathcal{L}(\Mat{W}_{k})$ in \eqref{eq:ACL_loss_by_class}. However, each class-specific loss $\mathcal{L}^{(y)}(\Mat{W}_{k})$ equals its sample count times its mean squared error, so head-class mean losses receive larger coefficients under class imbalance. Minimizing this ridge-regression objective \cite{RidgeRegression_Hoerl_1970} therefore places greater emphasis on reducing head-class mean errors, potentially biasing the classifier toward head classes. To rectify this imbalance, AIR introduces ARM to give class-mean losses equal weight.

    \subsection{Analytic Imbalance Rectification}\label{sec:ARM}
        \par AIR introduces ARM to equalize total sample weights across classes, as shown in \Cref{fig:flowchart} (c). ARM assigns a weight $\pi_{k,y}$ to the class-specific loss $\mathcal{L}^{(y)}(\Mat{W}_{k})$:
        \begin{equation}
            \begin{split}
                \mathcal{L}_{\text{we}}(\Mat{W}_{k})
                    &=\sum_{y=0}^{C_k-1}\pi_{k,y}\mathcal{L}^{(y)}(\Mat{W}_{k}) + \gamma\SFN{\Mat{W}_{k}}\\
                    &=\sum_{y=0}^{C_k-1}\pi_{k,y}\SFN{\Mat{X}_{1:k}^{(y)}\Mat{W}_{k} - \Mat{Y}_{1:k}^{(y)}} + \gamma\SFN{\Mat{W}_{k}}.
            \end{split}
        \end{equation}

        \par Let $N_{1:k}=\sum_{t=1}^{k}N_t$ be the number of training samples from phase 1 to $k$ and $N_{1:k}^{(y)}$ be the cumulative number of samples in class $y$. AIR uses the normalized inverse-frequency weight
        \begin{equation}\label{eq:pi}
            \pi_{k,y}=\frac{N_{1:k}}{C_kN_{1:k}^{(y)}}.
        \end{equation}
        Thus $N_{1:k}^{(y)}\pi_{k,y}=N_{1:k}/C_k$: every observed class has equal total sample weight, and the data term is $(N_{1:k}/C_k)\sum_y\mathcal{L}^{(y)}/N_{1:k}^{(y)}$, an equally weighted sum of class-mean losses. The total sample weight is preserved: $\sum_y N_{1:k}^{(y)}\pi_{k,y}=N_{1:k}$. If the observed classes are balanced, AIR reduces to the unweighted analytic classifier with the same $\gamma$ as $\pi_{k,y}=1$.

        \par The rectified classifier in \Cref{fig:flowchart} (d) is obtained iteratively from the following weighted ridge solution, described in \Cref{thm:recursive}.
        \begin{theorem}\label{thm:recursive}
            For $\gamma>0$, the unique minimizer of the weighted loss function $\mathcal{L}_{\text{\textrm{we}}}(\Mat{W}_{k})$ is
            \begin{equation}\label{eq:weighted_joint_solution}
                \Mat{\bar{W}}_{k} = \left(\sum_{y=0}^{C_k-1}\pi_{k,y}\Mat{A}_{1:k}^{(y)} + \gamma\Mat{I}\right)^{\inv}\left(\sum_{y=0}^{C_k-1}\Mat{C}_{1:k}^{(y)}\right)\Mat{\Pi}_k,
            \end{equation}
            where $\Mat{\Pi}_k = \operatorname{diag}(\pi_{k,0}, \pi_{k,1}, \dots, \pi_{k,C_k-1})$, and
            \begin{equation}\label{eq:AC}
                \begin{cases}
                    \Mat{A}_{1:k}^{(y)} = \Mat{X}_{1:k}^{(y)\T}\Mat{X}_{1:k}^{(y)} = \Mat{A}_{1:k-1}^{(y)} + \Mat{X}_{k}^{(y)\T}\Mat{X}_{k}^{(y)}\\
                    \Mat{C}_{1:k}^{(y)} = \Mat{X}_{1:k}^{(y)\T}\Mat{Y}_{1:k}^{(y)} = \Mat{C}_{1:k-1}^{(y)} + \Mat{X}_{k}^{(y)\T}\Mat{Y}_{k}^{(y)}.
                \end{cases}
            \end{equation}
        \end{theorem}
        \begin{proof}
            To minimize the loss function, we first calculate the gradient of the loss function $\mathcal{L}_{\text{we}}$, with respect to the weight:
            \begin{equation}
                \begin{aligned}
                    &\frac{\partial}{\partial\Mat{W}_{k}}\left(\sum_{y=0}^{C_k-1}\pi_{k,y}\SFN{\Mat{X}_{1:k}^{(y)}\Mat{W}_{k} - \Mat{Y}_{1:k}^{(y)}} + \gamma\SFN{\Mat{W}_{k}}\right) \\
                    &= -2\sum_{y=0}^{C_k-1}\pi_{k,y}\Mat{X}_{1:k}^{(y)\T}(\Mat{Y}_{1:k}^{(y)} - \Mat{X}_{1:k}^{(y)}\Mat{W}_{k}) + 2\gamma\Mat{W}_{k}.
                \end{aligned}
            \end{equation}
            Setting the gradient to the zero matrix yields the optimal weight:
            \begin{equation}\label{eq:proof_optimal_weight}
                \left(\sum_{y=0}^{C_k-1}\pi_{k,y}\Mat{X}_{1:k}^{(y)\T}\Mat{X}_{1:k}^{(y)} + \gamma\Mat{I}\right)\Mat{\bar{W}}_{k}
                = \sum_{y=0}^{C_k-1}\pi_{k,y}\Mat{X}_{1:k}^{(y)\T}\Mat{Y}_{1:k}^{(y)}.
            \end{equation}
            For single-label classification tasks, $\Mat{Y}_{1:k}^{(y)}$ is a matrix with one column of ones and zeros elsewhere. Therefore,
            \begin{equation}\label{eq:weighted_C}
                \sum_{y=0}^{C_k-1}\pi_{k,y}\Mat{X}_{1:k}^{(y)\T}\Mat{Y}_{1:k}^{(y)}
                = \sum_{y=0}^{C_k-1}\Mat{C}_{1:k}^{(y)}\Mat{\Pi}_k
                = (\sum_{y=0}^{C_k-1}\Mat{C}_{1:k}^{(y)})\Mat{\Pi}_k.
            \end{equation}
            Since each $\pi_{k,y}>0$ and $\gamma>0$, the weighted Gram matrix plus $\gamma\Mat{I}$ is positive definite; the loss is strictly convex and the stationary point is its unique minimizer. Combining \cref{eq:proof_optimal_weight,eq:weighted_C} completes the proof.
        \end{proof}
        \noindent \Cref{thm:recursive} shows that the classifier weight iteratively calculated by $\Mat{A}_{1:k}^{(y)}$ and $\Mat{C}_{1:k}^{(y)}$ with \eqref{eq:AC} is identical to the joint-learning one. We call this the \textit{weight-invariant property}, which enables exact class-weighted fitting without exemplar replay. \Cref{sec:Exp_WeightInvariance} numerically verifies this property. \Cref{algo:AIR} processes each sample once and recovers the weighted joint solution at phase ends using counts maintained online.

        \begin{algorithm}[ht]
            \caption{AIR for phase-disjoint CIL with phase-end classifier updates.}\label{algo:AIR}
            \footnotesize
            \algrenewcommand\algorithmicindent{0.8em}
            \begin{algorithmic}
                \Procedure{TrainForOnePhase}{$\dataset{D}_k$, $\gamma$, $\Mat{\Theta}$, $\Mat{Q}_{1:k-1}$, $\Mat{S}$}
                    \LComment{Maintain class/sample counts online; updates are omitted.}
                    \LComment{Phase-local $\Mat{G}^{(y)}$ starts at zero on first use; zero-pad $\Mat{S}$ for new classes.}
                    \ForAll {$(\Mat{\mathcal{X}}, y) \in \dataset{D}_k$}
                        \State $\Vec{x} \gets \mathcal{B}(f_\text{backbone}(\Mat{\mathcal{X}}, \Mat{\Theta}))$
                        \State $\Mat{G}^{(y)} \gets \Mat{G}^{(y)} + \Vec{x}^{\T}\Vec{x}$
                        \State $\Mat{S} \gets \Mat{S} + \Vec{x}^{\T}\onehot(y)$
                    \EndFor
                    \LComment{At phase end: $\Mat{G}^{(y)}=\Mat{A}_k^{(y)}$ and $\Mat{S}=\sum_{t=1}^{k}\Mat{C}_t$.}
                    \State $\Mat{Q}_{1:k} \gets \Mat{Q}_{1:k-1} + \sum_{y:N_k^{(y)}>0}\Mat{G}^{(y)}/N_k^{(y)}$
                    \LComment{$\Mat{Q}_{1:k}$ is the class-normalized Gram sum in \eqref{eq:ACL_tricks}.}
                    \State $s_k \gets N_{1:k}/C_k$; $\pi_{k,y}\gets s_k/N_{1:k}^{(y)}$ for $0\leq y<C_k$
                    \State $\Mat{\Pi}_k \gets \operatorname{diag}(\pi_{k,0}, \pi_{k,1}, \dots, \pi_{k,C_k-1})$
                    \LComment{Matches \eqref{eq:weighted_joint_solution} since $s_k\Mat{Q}_{1:k}=\sum_y\pi_{k,y}\Mat{A}_{1:k}^{(y)}$ and $\Mat{S}=\sum_y\Mat{C}_{1:k}^{(y)}$.}
                    \State $\Mat{\bar{W}}_{k} \gets (s_k\Mat{Q}_{1:k} + \gamma\Mat{I})^{\inv}\Mat{S}\Mat{\Pi}_k$
                    \State \Return $\Mat{\bar{W}}_{k}$, $\Mat{Q}_{1:k}$, $\Mat{S}$
                \EndProcedure
            \end{algorithmic}
        \end{algorithm}

    \subsection{Generalized AIR}\label{section:G-AIR}
        \par The compressed Gram statistic $\Mat{Q}_{1:k}$ in \Cref{algo:AIR} relies on classes being phase-disjoint. Each class's count becomes final at the end of its only phase, so
        \begin{equation}\label{eq:ACL_tricks}
            \Mat{Q}_{1:k}
            =\sum_{t=1}^{k}\sum_{y:N_t^{(y)}>0}\frac{\Mat{A}_{t}^{(y)}}{N_t^{(y)}}
            =\sum_{y=0}^{C_k-1}\frac{\Mat{A}_{1:k}^{(y)}}{N_{1:k}^{(y)}}.
        \end{equation}
        Thus $s_k\Mat{Q}_{1:k}=\sum_y\pi_{k,y}\Mat{A}_{1:k}^{(y)}$, and $\Mat{S}\Mat{\Pi}_k$ is the weighted cross-correlation in \Cref{thm:recursive}. This recovers the joint weighted-ridge solution at each phase end. Between phases, AIR retains one $f\times f$ matrix and $f\times C_k$ cross-correlation and classifier matrices, requiring $\Theta(f^2+fC_k)$ storage. Phase-local Gram matrices are discarded after aggregation. However, in GCIL, observing more samples from an old class changes its cumulative denominator $N_{1:k}^{(y)}$ and hence rescales its earlier contribution. Generalized AIR (G-AIR) therefore retains $\Mat{H}^{(y)}=\Mat{A}_{1:k}^{(y)}$ for every class and recomputes $\Mat{R}_k=\sum_y\pi_{k,y}\Mat{H}^{(y)}$ when the classifier is queried. G-AIR supports anytime inference: \cref{eq:weighted_joint_solution} recovers the current-prefix joint classifier from statistics and counts observed so far, without waiting for phase end. Its matrix storage cost is $\Theta(C_kf^2+fC_k)$. A further storage analysis is given in \Cref{tab:memory-analysis}. Its pseudo-code is listed in \Cref{algo:GAIR}.

        \begin{algorithm}[ht]
            \caption{G-AIR for GCIL.}\label{algo:GAIR}
            \footnotesize
            \algrenewcommand\algorithmicindent{0.8em}
            \begin{algorithmic}
                \Procedure{TrainForOnePhase}{$\dataset{D}_k$, $\gamma$, $\Mat{\Theta}$, $\{\Mat{H}^{(y)}\}$, $\Mat{S}$}
                    \LComment{Maintain class/sample counts online; updates are omitted.}
                    \LComment{Retain $\Mat{H}^{(y)}$ across phases; initialize only unseen classes to zero.}
                    \LComment{Zero-pad $\Mat{S}$ for new classes.}
                    \ForAll {$(\Mat{\mathcal{X}}, y) \in \dataset{D}_k$}
                        \State $\Vec{x} \gets \mathcal{B}(f_\text{backbone}(\Mat{\mathcal{X}}, \Mat{\Theta}))$
                        \State $\Mat{H}^{(y)} \gets \Mat{H}^{(y)} + \Vec{x}^{\T}\Vec{x}$
                        \State $\Mat{S} \gets \Mat{S} + \Vec{x}^{\T}\onehot(y)$
                    \EndFor
                    \LComment{At phase end: $\Mat{H}^{(y)}=\Mat{A}_{1:k}^{(y)}$ and $\Mat{S}=\sum_{t=1}^{k}\Mat{C}_t$.}
                    \State $s_k \gets N_{1:k}/C_k$; $\pi_{k,y}\gets s_k/N_{1:k}^{(y)}$ for $0\leq y<C_k$
                    \State $\Mat{R}_k \gets \sum_{y=0}^{C_k-1}\pi_{k,y}\Mat{H}^{(y)}$
                    \State $\Mat{\Pi}_k \gets \operatorname{diag}(\pi_{k,0}, \pi_{k,1}, \dots, \pi_{k,C_k-1})$
                    \LComment{Matches \eqref{eq:weighted_joint_solution} since $\Mat{R}_k=\sum_y\pi_{k,y}\Mat{A}_{1:k}^{(y)}$ and $\Mat{S}=\sum_y\Mat{C}_{1:k}^{(y)}$.}
                    \State $\Mat{\bar{W}}_{k} \gets (\Mat{R}_k + \gamma\Mat{I})^{\inv}\Mat{S}\Mat{\Pi}_k$
                    \State \Return $\Mat{\bar{W}}_{k}$, $\{\Mat{H}^{(y)}\}$, $\Mat{S}$
                \EndProcedure
            \end{algorithmic}
        \end{algorithm}

\section{Experiments}
    \newcommand{\LTAccuracyRows}{DGR       \cite{DGR_He_CVPR2024}                             &\ve{61.01}{0.45}&\ve{47.49}{0.77}&\ve{57.85}{0.24}&\ve{85.67}{0.27} &\ve{65.07}{0.44} &\ve{67.67}{3.45}&\ve{60.09}{1.66}&\ve{44.14}{1.54}&\ve{54.35}{0.77}&\ve{72.12}{0.35}&\ve{56.55}{0.32}&\ve{63.60}{1.31}&\ve{57.50}{0.39}&9.19\\
MVP-R     \cite{MVP_Moon_ICCV2023}                           &\ve{61.39}{0.66}&\ve{64.49}{0.57}&\ve{54.57}{0.59}&\ve{83.62}{0.25} &\ve{70.99}{0.68} &\ve{67.60}{2.88}&\ve{59.39}{2.88}&\ve{51.16}{0.58}&\ve{46.42}{0.50}&\ve{69.46}{0.22}&\ve{58.96}{0.36}&\ve{59.59}{1.33}&\ve{50.42}{0.70}&6.84\\
CLIB      \cite{CLIB_Koh_ICLR2022}                           &\ve{62.48}{0.55}&\ve{62.86}{1.08}&\ve{62.36}{0.79}&\ve{84.15}{0.30} &\ve{70.03}{0.26} &\ve{71.63}{2.41}&\ve{65.69}{1.50}&\ve{54.84}{0.59}&\ve{47.56}{0.40}&\ve{66.17}{0.29}&\ve{53.51}{0.34}&\ve{60.34}{1.00}&\ve{50.67}{0.64}&0.51\\
PRS       \cite{PRS_Kim_ECCV2020}                            &\ve{64.21}{0.61}&\ve{49.02}{0.67}&\ve{66.42}{0.77}&\ve{81.29}{0.59} &\ve{64.45}{0.60} &\ve{70.46}{2.53}&\ve{69.81}{1.97}&\ve{46.39}{0.68}&\ve{56.60}{1.72}&\ve{73.66}{0.39}&\ve{66.45}{0.50}&\ve{64.60}{1.30}&\ve{61.34}{0.89}&7.20\\
MEMO      \cite{MEMO_Zhou_ICLR2023}                          &\ve{65.03}{0.47}&\ve{61.29}{0.62}&\ve{66.64}{0.96}&\ve{88.09}{0.27} &\ve{76.37}{0.36} &\ve{74.76}{2.13}&\ve{71.41}{0.89}&\ve{49.89}{0.74}&\ve{50.35}{0.24}&\ve{67.60}{0.30}&\ve{59.59}{0.40}&\ve{59.66}{1.11}&\ve{54.64}{0.34}&4.59\\
MLG       \cite{MLG_Liang_PR2025}                            &\ve{65.84}{1.93}&\ve{66.33}{1.12}&\ve{69.94}{0.64}&\ve{67.52}{10.23}&\ve{51.68}{10.54}&\ve{76.00}{1.65}&\ve{69.87}{0.62}&\ve{54.70}{0.41}&\ve{56.52}{0.41}&\ve{77.06}{0.37}&\ve{62.90}{0.38}&\ve{71.22}{0.91}&\ve{66.36}{0.24}&29.14\\
BiC       \cite{BiC_Wu_CVPR2019}                             &\ve{69.98}{0.33}&\ve{69.60}{1.55}&\ve{73.13}{0.69}&\ve{88.05}{0.25} &\ve{72.65}{0.36} &\ve{74.54}{2.04}&\ve{70.75}{0.49}&\ve{55.56}{0.84}&\ve{64.32}{0.45}&\ve{76.35}{0.22}&\ve{63.20}{0.31}&\ve{67.41}{1.36}&\ve{64.17}{0.45}&7.43\\
ER        \cite{ER_Rolnick_NeurIPS2019}                      &\ve{70.18}{0.49}&\ve{65.35}{0.88}&\ve{63.03}{0.63}&\ve{89.42}{0.31} &\ve{77.91}{0.51} &\ve{73.17}{1.96}&\ve{67.32}{1.20}&\ve{59.56}{0.60}&\ve{60.23}{0.40}&\ve{79.44}{0.21}&\ve{71.02}{0.38}&\ve{70.72}{1.24}&\ve{64.94}{0.53}&6.96\\
DRC       \cite{DRC_Chen_ICCV2023}                           &\ve{62.78}{0.88}&\ve{47.60}{0.83}&\ve{45.82}{1.14}&\ve{83.08}{0.47} &\ve{68.39}{0.91} &\ve{65.71}{4.15}&\ve{50.63}{4.43}&\ve{50.87}{0.93}&\ve{64.38}{0.59}&\ve{76.52}{0.19}&\ve{66.28}{0.18}&\ve{67.60}{1.45}&\ve{66.44}{0.44}&7.02\\
LUCIR     \cite{LUCIR_Hou_CVPR2019}                          &\ve{69.88}{1.51}&\ve{69.21}{0.73}&\ve{73.77}{0.57}&\ve{90.16}{0.32} &\ve{81.73}{0.73} &\ve{64.87}{7.37}&\ve{60.11}{8.03}&\ve{57.64}{0.32}&\ve{65.02}{0.29}&\ve{78.76}{0.35}&\ve{72.85}{0.19}&\ve{64.33}{2.22}&\ve{60.06}{2.45}&9.11\\
UniCIL    \cite{UniCIL_Yang_arXiv2023}                       &\ve{74.20}{0.45}&\ve{73.31}{0.74}&\ve{75.56}{0.68}&\ve{86.88}{0.36} &\ve{74.26}{0.36} &\ve{74.56}{2.68}&\ve{75.84}{0.85}&\ve{70.62}{0.78}&\ve{62.47}{0.48}&\ve{80.64}{0.19}&\ve{70.28}{0.40}&\ve{74.58}{1.77}&\ve{71.44}{0.61}&9.06\\
LUCIR/2s  \cite{LUCIR_Hou_CVPR2019,LTCIL_Liu_ECCV2022}       &\ve{74.24}{0.54}&\ve{72.12}{0.74}&\ve{73.93}{0.60}&\ve{89.86}{0.39} &\ve{79.83}{0.78} &\ve{80.92}{1.93}&\ve{77.02}{1.26}&\ve{59.07}{0.54}&\ve{64.53}{0.30}&\ve{79.07}{0.15}&\ve{71.97}{0.12}&\ve{73.10}{1.39}&\ve{69.43}{0.74}&9.09\\
PODNet/2s \cite{PODNet_Douillard_ECCV2020,LTCIL_Liu_ECCV2022}&\ve{74.32}{0.38}&\ve{76.82}{0.66}&\ve{74.14}{0.47}&\ve{88.18}{0.47} &\ve{79.73}{0.81} &\ve{81.64}{1.48}&\ve{77.92}{0.71}&\ve{61.71}{0.79}&\ve{64.31}{0.41}&\ve{77.31}{0.32}&\ve{70.50}{0.22}&\ve{71.05}{1.13}&\ve{68.51}{0.53}&9.19\\
\midrule
LAE       \cite{LAE_Gao_ICCV2023}                            &\ve{59.28}{0.53}&\ve{66.68}{1.10}&\ve{64.38}{1.30}&\ve{79.73}{1.64} &\ve{58.40}{2.37} &\ve{72.00}{1.50}&\ve{63.91}{1.61}&\ve{45.50}{0.55}&\ve{49.54}{0.38}&\ve{65.90}{0.25}&\ve{43.99}{0.38}&\ve{55.46}{0.92}&\ve{45.88}{0.53}&3.83\\
DAP24     \cite{DAP_Hong_IJCAI2024}                          &\ve{61.97}{0.44}&\ve{66.71}{0.79}&\ve{57.56}{0.48}&\ve{80.86}{0.45} &\ve{65.64}{0.49} &\ve{71.07}{1.87}&\ve{64.58}{0.99}&\ve{52.80}{0.58}&\ve{48.00}{0.40}&\ve{67.80}{0.23}&\ve{54.44}{0.40}&\ve{61.36}{0.95}&\ve{52.80}{0.58}&5.21\\
GACL      \cite{GACL_Zhuang_NeurIPS2024}                     &\ve{64.51}{0.31}&\ve{74.01}{0.75}&\ve{63.16}{0.47}&\ve{84.08}{0.22} &\ve{63.16}{0.47} &\ve{70.66}{1.46}&\ve{63.16}{0.47}&\ve{60.20}{0.65}&\ve{53.80}{0.53}&\ve{72.39}{0.30}&\ve{53.80}{0.53}&\ve{61.92}{1.24}&\ve{53.80}{0.53}&2.56\\
DS-AL     \cite{DS-AL_Zhuang_AAAI2024}                       &\ve{64.90}{0.33}&\ve{73.90}{0.75}&\ve{62.66}{0.46}&\ve{84.70}{0.21} &\ve{64.41}{0.52} &\ve{70.97}{1.51}&\ve{63.63}{0.56}&\ve{59.94}{0.66}&\ve{52.98}{0.52}&\ve{73.11}{0.29}&\ve{55.38}{0.51}&\ve{62.74}{1.20}&\ve{54.43}{0.53}&2.79\\
SLDA      \cite{SLDA_Hayes_CVPR2020W}                        &\ve{67.57}{0.39}&\ve{64.67}{1.02}&\ve{71.02}{0.51}&\ve{87.10}{0.20} &\ve{71.03}{0.51} &\ve{73.59}{2.12}&\ve{71.02}{0.52}&\ve{56.42}{0.81}&\ve{60.06}{0.33}&\ve{71.76}{0.31}&\ve{60.07}{0.33}&\ve{64.06}{1.29}&\ve{60.07}{0.33}&2.47\\
RanPAC    \cite{RanPAC_McDonnell_NeurIPS2023}                &\ve{67.85}{0.70}&\ve{73.92}{0.71}&\ve{63.11}{0.52}&\ve{86.39}{0.30} &\ve{68.36}{0.50} &\ve{72.34}{2.09}&\ve{64.87}{1.28}&\ve{61.14}{0.70}&\ve{55.02}{0.50}&\Ve{77.08}{0.16}&\ve{63.92}{0.30}&\ve{68.00}{2.58}&\ve{60.00}{2.31}&2.22\\
MOS       \cite{MOS_Sun_AAAI2025}                            &\ve{68.03}{0.41}&\Ve{77.54}{0.74}&\ve{74.54}{0.54}&\ve{85.27}{0.55} &\ve{73.26}{0.84} &\Ve{79.98}{1.54}&\ve{74.54}{0.60}&\ve{57.65}{0.61}&\ve{51.59}{0.27}&\ve{68.23}{0.34}&\ve{56.69}{0.45}&\ve{63.45}{0.79}&\ve{53.56}{0.67}&4.74\\
CODA-P    \cite{CODA-Prompt_Smith_CVPR2023}                  &\ve{68.85}{0.46}&\ve{71.20}{1.04}&\ve{67.89}{0.68}&\ve{85.00}{0.44} &\ve{70.32}{0.55} &\ve{75.40}{1.53}&\ve{69.45}{0.76}&\ve{57.86}{0.70}&\ve{58.26}{0.30}&\ve{75.06}{0.38}&\ve{63.83}{0.31}&\ve{69.15}{0.80}&\ve{62.81}{0.36}&7.13\\
FeCAM     \cite{FeCAM_Goswami_NeurIPS2023}                   &\ve{68.98}{0.49}&\ve{74.56}{0.81}&\ve{70.17}{0.61}&\ve{86.44}{0.28} &\ve{71.31}{0.59} &\ve{75.75}{1.83}&\ve{70.54}{0.68}&\ve{55.34}{0.84}&\ve{54.54}{0.47}&\ve{76.18}{0.73}&\ve{64.48}{0.99}&\ve{67.34}{2.11}&\ve{61.05}{1.77}&1.87\\
LwF       \cite{LwF_Li_TPAMI2018}                            &\ve{69.13}{0.53}&\ve{68.94}{0.84}&\ve{67.67}{0.76}&\ve{84.95}{0.64} &\ve{69.58}{1.02} &\ve{76.81}{1.73}&\ve{70.30}{1.03}&\ve{58.52}{0.74}&\ve{61.11}{0.29}&\ve{75.78}{0.22}&\ve{63.40}{0.13}&\ve{69.01}{1.21}&\ve{63.47}{0.41}&8.96\\
SimpleCIL \cite{SimpleCIL_Zhou_IJCV2025}                     &\ve{69.19}{0.37}&\Ve{77.54}{0.74}&\ve{74.57}{0.53}&\ve{86.41}{0.19} &\ve{74.57}{0.53} &\ve{79.69}{1.51}&\ve{74.57}{0.53}&\ve{59.27}{0.60}&\ve{57.18}{0.25}&\ve{67.95}{0.30}&\ve{57.18}{0.25}&\ve{64.12}{0.59}&\ve{57.18}{0.25}&2.44\\
EWC       \cite{EWC_Kirkpatrick_PNAS2017}                    &\ve{69.21}{0.48}&\ve{68.91}{0.72}&\ve{68.58}{0.49}&\ve{85.51}{0.50} &\ve{70.38}{0.73} &\ve{75.88}{1.76}&\ve{69.65}{0.76}&\ve{58.89}{0.48}&\ve{61.46}{0.13}&\ve{75.95}{0.28}&\ve{63.32}{0.35}&\ve{68.58}{1.25}&\ve{63.41}{0.27}&7.71\\
AEF-OCL   \cite{AEF-OCL_Zhuang_TVT2025}                      &\ve{69.85}{0.32}&\ve{76.91}{0.72}&\Ve{74.66}{0.47}&\Ve{87.59}{0.11} &\ve{72.79}{0.44} &\ve{77.60}{1.86}&\ve{73.53}{0.32}&\ve{64.15}{0.64}&\ve{60.39}{0.32}&\ve{70.83}{0.30}&\ve{57.08}{0.38}&\ve{66.06}{0.74}&\ve{56.66}{0.41}&2.63\\
APER      \cite{SimpleCIL_Zhou_IJCV2025}                     &\ve{70.68}{0.43}&\Ve{77.54}{0.73}&\ve{74.56}{0.53}&\ve{86.85}{0.26} &\Ve{75.25}{0.53} &\ve{79.87}{1.52}&\Ve{74.79}{0.57}&\ve{59.27}{0.60}&\ve{57.18}{0.25}&\ve{72.30}{0.47}&\ve{62.11}{0.38}&\ve{67.69}{1.06}&\ve{60.76}{1.04}&1.86\\
APART     \cite{APART_Qi_ML2025}                             &\Ve{71.35}{0.36}&\ve{71.84}{1.12}&\ve{69.97}{0.36}&\ve{84.07}{0.85} &\ve{63.72}{1.69} &\ve{76.45}{1.77}&\ve{69.69}{0.97}&\Ve{64.28}{0.54}&\VE{65.76}{0.33}&\VE{79.82}{0.19}&\VE{67.66}{0.29}&\VE{74.44}{0.94}&\VE{68.44}{0.27}&11.20\\
\midrule
AIR                                                          &\VE{74.56}{0.31}&\VE{79.01}{0.53}&\VE{78.65}{0.37}&\VE{89.09}{0.15} &\VE{78.65}{0.37} &\VE{82.96}{1.25}&\VE{78.65}{0.37}&\VE{66.12}{0.70}&\Ve{64.87}{0.34}&\ve{75.63}{0.33}&\Ve{64.87}{0.34}&\Ve{71.31}{0.68}&\Ve{64.87}{0.34}&3.00\\
}
    \newcommand{\LTMacroFOneRows}{DGR       \cite{DGR_He_CVPR2024}                             &\ve{59.41}{0.47}&\ve{44.42}{0.90}&\ve{55.22}{0.27}&\ve{85.39}{0.27} &\ve{62.13}{0.49} &\ve{65.24}{3.64}&\ve{57.63}{1.45}&\ve{42.68}{1.37}&\ve{52.31}{0.78}&\ve{71.66}{0.42}&\ve{56.38}{0.24}&\ve{63.00}{1.44}&\ve{56.81}{0.41}&9.19\\
MVP-R     \cite{MVP_Moon_ICCV2023}                           &\ve{61.08}{0.74}&\ve{62.66}{0.76}&\ve{51.87}{0.47}&\ve{83.02}{0.25} &\ve{67.98}{0.70} &\ve{64.74}{3.35}&\ve{57.57}{2.91}&\ve{53.49}{0.52}&\ve{47.10}{0.56}&\ve{71.13}{0.27}&\ve{59.36}{0.37}&\ve{61.12}{1.24}&\ve{52.94}{0.77}&6.84\\
CLIB      \cite{CLIB_Koh_ICLR2022}                           &\ve{63.91}{0.63}&\ve{61.44}{1.30}&\ve{64.25}{0.71}&\ve{84.02}{0.30} &\ve{67.87}{0.35} &\ve{69.49}{2.96}&\ve{65.64}{1.51}&\ve{58.17}{0.58}&\ve{54.47}{0.52}&\ve{68.17}{0.36}&\ve{55.13}{0.29}&\ve{62.63}{0.92}&\ve{55.59}{0.55}&0.51\\
PRS       \cite{PRS_Kim_ECCV2020}                            &\ve{64.04}{0.65}&\ve{47.39}{0.77}&\ve{66.78}{0.66}&\ve{79.91}{0.63} &\ve{59.03}{0.79} &\ve{68.05}{2.85}&\ve{69.37}{1.98}&\ve{49.75}{0.71}&\ve{61.61}{1.39}&\ve{73.86}{0.38}&\ve{62.43}{0.26}&\ve{66.66}{1.17}&\ve{63.63}{0.67}&7.20\\
MLG       \cite{MLG_Liang_PR2025}                            &\ve{64.27}{2.03}&\ve{65.32}{1.28}&\ve{67.98}{0.85}&\ve{65.65}{10.79}&\ve{47.92}{11.34}&\ve{74.00}{1.82}&\ve{67.58}{0.75}&\ve{55.86}{0.53}&\ve{55.97}{0.49}&\ve{75.87}{0.45}&\ve{60.47}{0.38}&\ve{70.05}{0.84}&\ve{64.51}{0.33}&29.14\\
MEMO      \cite{MEMO_Zhou_ICLR2023}                          &\ve{65.48}{0.56}&\ve{58.72}{0.90}&\ve{66.19}{0.96}&\ve{88.01}{0.27} &\ve{74.91}{0.48} &\ve{72.49}{2.58}&\ve{70.10}{1.02}&\ve{52.77}{0.78}&\ve{55.29}{0.28}&\ve{68.58}{0.25}&\ve{59.31}{0.41}&\ve{61.88}{1.00}&\ve{57.50}{0.43}&4.59\\
BiC       \cite{BiC_Wu_CVPR2019}                             &\ve{68.47}{0.40}&\ve{66.83}{1.88}&\ve{70.58}{0.93}&\ve{87.85}{0.26} &\ve{69.80}{0.32} &\ve{71.19}{2.59}&\ve{67.57}{0.50}&\ve{56.04}{0.83}&\ve{63.41}{0.52}&\ve{75.88}{0.28}&\ve{62.66}{0.35}&\ve{66.51}{1.36}&\ve{63.38}{0.58}&7.43\\
ER        \cite{ER_Rolnick_NeurIPS2019}                      &\ve{68.88}{0.56}&\ve{62.94}{1.09}&\ve{61.06}{0.66}&\ve{89.12}{0.33} &\ve{75.76}{0.69} &\ve{69.83}{2.40}&\ve{63.63}{1.31}&\ve{60.78}{0.64}&\ve{61.32}{0.39}&\ve{78.73}{0.29}&\ve{69.64}{0.41}&\ve{70.06}{1.03}&\ve{63.68}{0.57}&6.96\\
DRC       \cite{DRC_Chen_ICCV2023}                           &\ve{62.51}{0.74}&\ve{48.38}{0.78}&\ve{48.87}{1.11}&\ve{82.43}{0.50} &\ve{65.21}{1.05} &\ve{65.92}{3.81}&\ve{54.35}{3.70}&\ve{50.56}{0.84}&\ve{64.49}{0.53}&\ve{75.28}{0.33}&\ve{63.46}{0.35}&\ve{66.16}{1.36}&\ve{65.06}{0.33}&7.02\\
LUCIR     \cite{LUCIR_Hou_CVPR2019}                          &\ve{68.03}{1.57}&\ve{64.35}{0.75}&\ve{71.18}{0.60}&\ve{89.95}{0.38} &\ve{80.73}{0.93} &\ve{62.06}{7.60}&\ve{57.35}{8.34}&\ve{54.20}{0.40}&\ve{63.42}{0.37}&\ve{77.68}{0.60}&\ve{71.54}{0.35}&\ve{63.99}{2.20}&\ve{59.91}{2.45}&9.11\\
LUCIR/2s  \cite{LUCIR_Hou_CVPR2019,LTCIL_Liu_ECCV2022}       &\ve{73.27}{0.57}&\ve{68.12}{0.85}&\ve{72.15}{0.64}&\ve{89.77}{0.41} &\ve{78.75}{0.95} &\ve{79.46}{2.23}&\ve{75.59}{1.46}&\ve{57.75}{0.68}&\ve{64.61}{0.35}&\ve{78.70}{0.32}&\ve{71.53}{0.21}&\ve{73.34}{1.20}&\ve{69.43}{0.60}&9.09\\
UniCIL    \cite{UniCIL_Yang_arXiv2023}                       &\ve{73.57}{0.52}&\ve{70.77}{0.92}&\ve{75.40}{0.67}&\ve{86.19}{0.34} &\ve{71.23}{0.59} &\ve{71.21}{3.18}&\ve{74.31}{0.87}&\ve{72.56}{0.61}&\ve{67.87}{0.50}&\ve{79.58}{0.22}&\ve{68.34}{0.42}&\ve{73.82}{1.75}&\ve{71.49}{0.57}&9.06\\
PODNet/2s \cite{PODNet_Douillard_ECCV2020,LTCIL_Liu_ECCV2022}&\ve{73.98}{0.43}&\ve{75.46}{0.85}&\ve{72.91}{0.40}&\ve{88.07}{0.49} &\ve{78.90}{0.94} &\ve{80.34}{1.75}&\ve{76.71}{0.87}&\ve{63.76}{0.71}&\ve{65.03}{0.35}&\ve{77.10}{0.40}&\ve{69.77}{0.11}&\ve{71.52}{0.92}&\ve{68.26}{0.50}&9.19\\
\midrule
LAE       \cite{LAE_Gao_ICCV2023}                            &\ve{57.57}{0.50}&\ve{65.01}{1.20}&\ve{62.16}{1.09}&\ve{79.34}{1.46} &\ve{56.48}{1.96} &\ve{69.92}{1.68}&\ve{61.98}{1.63}&\ve{45.07}{0.63}&\ve{47.30}{0.39}&\ve{65.25}{0.37}&\ve{41.61}{0.52}&\ve{53.34}{0.92}&\ve{43.40}{0.53}&3.83\\
DAP24     \cite{DAP_Hong_IJCAI2024}                          &\ve{61.95}{0.42}&\ve{65.66}{0.84}&\ve{55.23}{0.45}&\ve{80.82}{0.44} &\ve{65.70}{0.67} &\ve{69.60}{2.24}&\ve{63.73}{1.27}&\ve{53.99}{0.60}&\ve{46.76}{0.39}&\ve{68.38}{0.34}&\ve{56.24}{0.29}&\ve{62.63}{0.64}&\ve{54.72}{0.47}&5.21\\
GACL      \cite{GACL_Zhuang_NeurIPS2024}                     &\ve{62.49}{0.35}&\ve{71.65}{0.97}&\ve{59.52}{0.61}&\ve{83.36}{0.21} &\ve{59.52}{0.61} &\ve{67.27}{1.76}&\ve{59.52}{0.61}&\ve{59.72}{0.69}&\ve{52.04}{0.52}&\ve{72.36}{0.43}&\ve{52.04}{0.52}&\ve{60.88}{1.18}&\ve{52.04}{0.52}&2.56\\
DS-AL     \cite{DS-AL_Zhuang_AAAI2024}                       &\ve{62.91}{0.37}&\ve{71.55}{0.97}&\ve{59.01}{0.59}&\ve{84.02}{0.21} &\ve{60.91}{0.67} &\ve{67.65}{1.82}&\ve{60.09}{0.71}&\ve{59.55}{0.72}&\ve{51.15}{0.50}&\ve{73.00}{0.43}&\ve{53.58}{0.55}&\ve{61.82}{1.08}&\ve{52.64}{0.51}&2.79\\
RanPAC    \cite{RanPAC_McDonnell_NeurIPS2023}                &\ve{66.49}{0.67}&\ve{71.53}{0.93}&\ve{59.52}{0.63}&\ve{86.01}{0.27} &\ve{65.87}{0.60} &\ve{69.18}{2.40}&\ve{61.46}{1.43}&\ve{60.94}{0.74}&\ve{53.92}{0.49}&\Ve{77.33}{0.20}&\ve{64.04}{0.39}&\ve{68.12}{2.30}&\ve{59.99}{1.96}&2.22\\
CODA-P    \cite{CODA-Prompt_Smith_CVPR2023}                  &\ve{67.77}{0.52}&\ve{69.65}{1.21}&\ve{65.27}{0.82}&\ve{84.52}{0.46} &\ve{68.15}{0.75} &\ve{73.49}{1.78}&\ve{67.06}{0.95}&\ve{58.91}{0.52}&\ve{57.74}{0.37}&\ve{74.63}{0.50}&\ve{63.36}{0.25}&\ve{68.53}{0.69}&\ve{61.99}{0.49}&7.13\\
LwF       \cite{LwF_Li_TPAMI2018}                            &\ve{68.22}{0.57}&\ve{67.14}{0.76}&\ve{65.21}{0.86}&\ve{84.51}{0.60} &\ve{67.64}{1.10} &\ve{75.15}{1.90}&\ve{68.48}{1.18}&\ve{59.17}{0.72}&\ve{60.40}{0.36}&\ve{75.74}{0.36}&\ve{63.58}{0.24}&\ve{68.62}{1.16}&\ve{62.98}{0.38}&8.96\\
EWC       \cite{EWC_Kirkpatrick_PNAS2017}                    &\ve{68.29}{0.53}&\ve{67.17}{0.91}&\ve{66.11}{0.59}&\ve{85.18}{0.51} &\ve{68.54}{0.88} &\ve{73.97}{2.01}&\ve{67.61}{0.85}&\ve{59.56}{0.52}&\ve{60.96}{0.31}&\ve{75.96}{0.30}&\ve{63.43}{0.41}&\ve{68.08}{1.20}&\ve{62.95}{0.40}&7.71\\
SLDA      \cite{SLDA_Hayes_CVPR2020W}                        &\ve{68.70}{0.39}&\ve{63.90}{1.31}&\ve{69.49}{0.60}&\ve{87.15}{0.15} &\ve{69.51}{0.60} &\ve{71.59}{2.39}&\ve{69.50}{0.61}&\ve{60.67}{0.79}&\VE{63.68}{0.38}&\ve{74.46}{0.45}&\ve{63.69}{0.39}&\ve{67.05}{1.08}&\Ve{63.69}{0.39}&2.47\\
FeCAM     \cite{FeCAM_Goswami_NeurIPS2023}                   &\ve{68.81}{0.48}&\ve{73.14}{1.02}&\ve{68.31}{0.76}&\ve{86.26}{0.25} &\ve{69.53}{0.73} &\ve{73.87}{2.13}&\ve{68.70}{0.83}&\ve{57.75}{0.92}&\ve{56.56}{0.55}&\ve{76.43}{0.56}&\Ve{64.83}{0.95}&\ve{68.31}{1.76}&\ve{62.04}{1.34}&1.87\\
AEF-OCL   \cite{AEF-OCL_Zhuang_TVT2025}                      &\ve{69.04}{0.39}&\ve{74.68}{0.96}&\ve{72.33}{0.65}&\Ve{87.30}{0.11} &\ve{70.37}{0.65} &\ve{75.09}{2.40}&\ve{71.32}{0.55}&\Ve{64.53}{0.61}&\ve{60.91}{0.37}&\ve{71.07}{0.40}&\ve{57.15}{0.54}&\ve{66.50}{0.68}&\ve{57.28}{0.56}&2.63\\
APART     \cite{APART_Qi_ML2025}                             &\ve{69.38}{0.43}&\ve{69.70}{1.30}&\ve{66.77}{0.55}&\ve{83.35}{0.91} &\ve{61.32}{1.56} &\ve{73.97}{2.11}&\ve{66.60}{1.22}&\ve{63.83}{0.60}&\Ve{63.49}{0.51}&\VE{78.80}{0.32}&\VE{65.79}{0.37}&\VE{72.83}{0.96}&\VE{66.07}{0.35}&11.20\\
MOS       \cite{MOS_Sun_AAAI2025}                            &\ve{69.56}{0.47}&\Ve{76.30}{0.91}&\ve{73.34}{0.69}&\ve{85.20}{0.55} &\ve{72.03}{0.98} &\Ve{78.74}{1.74}&\ve{73.28}{0.72}&\ve{61.09}{0.46}&\ve{56.88}{0.19}&\ve{70.89}{0.40}&\ve{61.79}{0.37}&\ve{66.33}{0.62}&\ve{58.87}{0.46}&4.74\\
SimpleCIL \cite{SimpleCIL_Zhou_IJCV2025}                     &\ve{69.98}{0.41}&\Ve{76.30}{0.91}&\Ve{73.37}{0.68}&\ve{86.34}{0.17} &\ve{73.37}{0.68} &\ve{78.46}{1.75}&\ve{73.37}{0.68}&\ve{62.21}{0.44}&\ve{60.00}{0.29}&\ve{70.04}{0.41}&\ve{60.00}{0.29}&\ve{66.30}{0.42}&\ve{60.00}{0.29}&2.44\\
APER      \cite{SimpleCIL_Zhou_IJCV2025}                     &\Ve{71.25}{0.42}&\Ve{76.30}{0.91}&\ve{73.36}{0.67}&\ve{86.78}{0.24} &\Ve{74.07}{0.67} &\ve{78.65}{1.76}&\Ve{73.59}{0.70}&\ve{62.20}{0.44}&\ve{59.99}{0.29}&\ve{73.60}{0.36}&\ve{64.08}{0.37}&\ve{69.38}{0.81}&\ve{62.97}{0.75}&1.86\\
\midrule
AIR                                                          &\VE{73.39}{0.35}&\VE{78.12}{0.57}&\VE{77.17}{0.51}&\VE{88.83}{0.15} &\VE{77.17}{0.51} &\VE{81.63}{1.50}&\VE{77.17}{0.51}&\VE{65.58}{0.65}&\ve{63.34}{0.39}&\ve{74.75}{0.43}&\ve{63.34}{0.39}&\Ve{70.22}{0.67}&\ve{63.34}{0.39}&3.00\\
}
    \newcommand{\SiBlurryAccuracyRows}{\SetCell[r=2]{l,m} O-LoRA   \cite{Online-LoRA_Wei_WACV2025}&\ascore&\ve{54.95}{0.98}&\ve{68.15}{2.52}&\ve{68.02}{1.83}&\ve{68.20}{2.69}&\ve{61.45}{0.98} &\ve{60.85}{1.06} &\ve{53.54}{1.70} &\ve{47.06}{2.17} &\ve{50.73}{3.33} &\ve{50.84}{3.78} &\ve{40.54}{0.71}&\ve{47.50}{1.66}&\ve{42.60}{1.07}&\SetCell[r=2]{c,m} 5.10\\
{}                                                             &\fscore&\ve{50.54}{1.14}&\ve{64.44}{3.07}&\ve{65.09}{2.25}&\ve{65.83}{2.98}&\ve{59.57}{1.18} &\ve{59.49}{1.14} &\ve{53.91}{1.50} &\ve{40.23}{2.62} &\ve{43.44}{4.13} &\ve{44.41}{4.56} &\ve{33.30}{0.95}&\ve{40.97}{2.17}&\ve{35.80}{1.43}&{}\\
\SetCell[r=2]{l,m} UniCIL   \cite{UniCIL_Yang_arXiv2023}       &\ascore&\ve{57.39}{4.70}&\ve{78.00}{0.97}&\ve{72.96}{1.45}&\ve{85.38}{0.34}&\ve{37.34}{15.94}&\ve{38.28}{16.41}&\ve{38.83}{16.87}&\ve{49.06}{16.10}&\ve{49.84}{16.21}&\ve{49.93}{16.92}&\ve{62.77}{2.05}&\ve{63.29}{2.97}&\ve{63.02}{4.83}&\SetCell[r=2]{c,m} 7.95\\
{}                                                             &\fscore&\ve{54.42}{4.78}&\ve{75.27}{1.47}&\ve{69.64}{1.60}&\ve{85.35}{0.32}&\ve{35.47}{15.85}&\ve{36.26}{16.20}&\ve{37.91}{16.93}&\ve{46.49}{16.42}&\ve{47.47}{16.76}&\ve{48.78}{17.38}&\ve{56.09}{2.53}&\ve{56.39}{3.43}&\ve{57.90}{5.94}&{}\\
\SetCell[r=2]{l,m} EWC++    \cite{EWC_Kirkpatrick_PNAS2017}    &\ascore&\ve{72.55}{0.48}&\ve{72.44}{1.22}&\ve{70.95}{0.86}&\ve{69.91}{0.79}&\ve{61.96}{0.76} &\ve{61.62}{1.19} &\ve{54.16}{0.58} &\ve{81.28}{0.68} &\ve{81.81}{0.67} &\ve{82.54}{1.29} &\ve{77.49}{1.02}&\ve{81.10}{0.98}&\ve{75.28}{0.51}&\SetCell[r=2]{c,m} 0.50\\
{}                                                             &\fscore&\ve{72.30}{0.48}&\ve{71.53}{1.43}&\ve{70.21}{1.07}&\ve{70.35}{0.75}&\ve{62.74}{0.63} &\ve{62.41}{0.92} &\ve{55.88}{0.52} &\ve{79.92}{0.82} &\ve{80.87}{0.80} &\ve{82.54}{1.23} &\ve{76.12}{1.12}&\ve{80.22}{1.04}&\ve{74.75}{0.74}&{}\\
\SetCell[r=2]{l,m} RM       \cite{RM_Bang_CVPR2021}            &\ascore&\ve{72.73}{0.58}&\ve{73.22}{1.34}&\ve{84.08}{0.46}&\ve{83.16}{0.17}&\ve{64.01}{1.48} &\ve{64.45}{2.10} &\ve{58.46}{1.28} &\ve{81.72}{1.84} &\ve{85.85}{1.00} &\ve{85.32}{1.13} &\ve{42.95}{3.23}&\ve{80.22}{1.61}&\ve{69.38}{1.16}&\SetCell[r=2]{c,m} 5.03\\
{}                                                             &\fscore&\ve{71.59}{0.61}&\ve{70.41}{1.77}&\ve{83.56}{0.52}&\ve{82.87}{0.19}&\ve{63.95}{1.29} &\ve{63.60}{1.91} &\ve{57.72}{1.18} &\ve{79.95}{2.26} &\ve{85.45}{1.12} &\ve{85.15}{1.11} &\ve{37.98}{3.05}&\ve{79.51}{1.70}&\ve{68.90}{1.16}&{}\\
\SetCell[r=2]{l,m} ER       \cite{ER_Rolnick_NeurIPS2019}      &\ascore&\ve{73.26}{0.48}&\ve{73.24}{1.27}&\ve{71.85}{0.97}&\ve{71.02}{1.07}&\ve{63.15}{0.67} &\ve{62.82}{1.22} &\ve{55.63}{0.71} &\ve{81.60}{0.69} &\ve{81.85}{0.57} &\ve{82.22}{1.40} &\ve{78.49}{1.41}&\ve{81.50}{0.91}&\ve{75.68}{0.38}&\SetCell[r=2]{c,m} 0.45\\
{}                                                             &\fscore&\ve{73.02}{0.47}&\ve{72.27}{1.48}&\ve{71.13}{1.21}&\ve{71.55}{1.01}&\ve{63.85}{0.50} &\ve{63.75}{0.97} &\ve{57.63}{0.54} &\ve{80.30}{0.86} &\ve{80.77}{0.76} &\ve{82.16}{1.36} &\ve{77.18}{1.50}&\ve{80.59}{0.95}&\ve{75.12}{0.36}&{}\\
\SetCell[r=2]{l,m} CLIB     \cite{CLIB_Koh_ICLR2022}           &\ascore&\ve{74.13}{0.45}&\ve{74.91}{1.30}&\ve{74.11}{0.90}&\ve{70.70}{0.94}&\ve{62.95}{0.72} &\ve{62.05}{1.16} &\ve{51.97}{0.51} &\ve{81.92}{0.32} &\ve{81.81}{0.65} &\ve{78.75}{1.57} &\ve{84.24}{0.77}&\ve{85.31}{0.86}&\ve{80.85}{0.32}&\SetCell[r=2]{c,m} 0.59\\
{}                                                             &\fscore&\ve{74.41}{0.44}&\ve{74.41}{1.48}&\ve{73.91}{1.00}&\ve{71.21}{0.96}&\ve{64.22}{0.61} &\ve{63.84}{0.80} &\ve{55.69}{0.43} &\ve{81.05}{0.42} &\ve{81.32}{0.69} &\ve{78.91}{1.46} &\ve{83.17}{0.98}&\ve{84.65}{0.99}&\ve{80.54}{0.39}&{}\\
\SetCell[r=2]{l,m} MVP-R    \cite{MVP_Moon_ICCV2023}           &\ascore&\ve{77.84}{0.34}&\ve{78.62}{1.16}&\ve{77.39}{1.00}&\ve{78.21}{0.74}&\ve{69.81}{0.61} &\ve{69.85}{0.85} &\ve{64.12}{0.61} &\ve{84.59}{0.71} &\ve{84.66}{0.89} &\ve{87.29}{0.78} &\ve{79.91}{1.07}&\ve{82.18}{1.15}&\ve{77.44}{0.53}&\SetCell[r=2]{c,m} 5.43\\
{}                                                             &\fscore&\ve{77.34}{0.33}&\ve{77.31}{1.36}&\ve{76.29}{1.21}&\ve{78.05}{0.75}&\ve{69.88}{0.50} &\ve{70.25}{0.63} &\ve{65.57}{0.37} &\ve{83.46}{0.87} &\ve{83.67}{1.14} &\ve{87.28}{0.78} &\ve{78.47}{1.17}&\ve{80.98}{1.34}&\ve{76.83}{0.55}&{}\\
\midrule
\SetCell[r=2]{l,m} Finetune                                    &\ascore&\ve{54.36}{0.78}&\ve{63.92}{2.04}&\ve{62.68}{1.51}&\ve{62.62}{2.24}&\ve{51.99}{0.74} &\ve{52.70}{0.68} &\ve{47.68}{0.75} &\ve{61.72}{2.05} &\ve{61.59}{1.66} &\ve{50.79}{3.79} &\ve{53.95}{1.02}&\ve{49.98}{2.18}&\ve{32.72}{2.91}&\SetCell[r=2]{c,m} 0.68\\
{}                                                             &\fscore&\ve{50.39}{0.84}&\ve{60.65}{2.52}&\ve{60.04}{1.85}&\ve{60.95}{2.44}&\ve{51.46}{0.75} &\ve{52.79}{0.99} &\ve{49.05}{0.51} &\ve{55.08}{2.50} &\ve{56.33}{2.26} &\ve{42.72}{4.53} &\ve{46.96}{1.44}&\ve{43.07}{2.19}&\ve{25.54}{2.53}&{}\\
\SetCell[r=2]{l,m} LwF      \cite{LwF_Li_TPAMI2018}            &\ascore&\ve{54.78}{0.80}&\ve{63.58}{2.13}&\ve{62.11}{1.75}&\ve{62.63}{2.29}&\ve{56.29}{0.82} &\ve{55.34}{0.67} &\ve{46.99}{1.12} &\ve{61.62}{1.99} &\ve{61.34}{1.60} &\ve{50.28}{3.67} &\ve{54.04}{1.09}&\ve{50.04}{2.15}&\ve{33.10}{2.91}&\SetCell[r=2]{c,m} 0.69\\
{}                                                             &\fscore&\ve{50.86}{0.86}&\ve{60.19}{2.65}&\ve{59.36}{2.15}&\ve{61.05}{2.43}&\ve{56.25}{0.80} &\ve{55.58}{0.62} &\ve{48.74}{0.62} &\ve{54.95}{2.44} &\ve{55.97}{2.18} &\ve{42.23}{4.31} &\ve{47.08}{1.47}&\ve{43.08}{2.15}&\ve{25.85}{2.56}&{}\\
\SetCell[r=2]{l,m} Dual-P   \cite{DualPrompt_Wang_ECCV2022}    &\ascore&\ve{58.08}{0.68}&\ve{70.88}{1.69}&\ve{70.73}{1.12}&\ve{68.38}{1.90}&\ve{58.55}{0.73} &\ve{58.03}{0.66} &\ve{51.07}{1.18} &\ve{63.18}{2.11} &\ve{63.27}{1.59} &\ve{52.60}{3.57} &\ve{55.31}{1.11}&\ve{51.12}{2.15}&\ve{33.79}{2.97}&\SetCell[r=2]{c,m} 4.45\\
{}                                                             &\fscore&\ve{54.47}{0.74}&\ve{68.46}{2.03}&\ve{68.87}{1.35}&\ve{66.48}{2.10}&\ve{58.17}{0.78} &\ve{57.97}{0.79} &\ve{52.67}{0.66} &\ve{57.13}{2.66} &\ve{58.60}{2.13} &\ve{45.33}{4.31} &\ve{48.64}{1.46}&\ve{44.50}{2.14}&\ve{26.88}{2.64}&{}\\
\SetCell[r=2]{l,m} MVP      \cite{MVP_Moon_ICCV2023}           &\ascore&\ve{58.41}{0.81}&\ve{69.38}{1.88}&\ve{69.11}{1.40}&\ve{68.42}{2.10}&\ve{60.24}{0.80} &\ve{58.88}{1.04} &\ve{49.28}{1.37} &\ve{63.98}{2.09} &\ve{63.89}{1.66} &\ve{52.90}{3.99} &\ve{55.70}{1.30}&\ve{52.17}{1.75}&\ve{36.97}{2.57}&\SetCell[r=2]{c,m} 5.45\\
{}                                                             &\fscore&\ve{54.55}{0.85}&\ve{66.51}{2.26}&\ve{66.90}{1.67}&\ve{66.73}{2.32}&\ve{59.74}{0.93} &\ve{58.65}{0.70} &\ve{50.18}{0.74} &\ve{57.83}{2.58} &\ve{59.04}{2.13} &\ve{45.28}{4.77} &\ve{48.85}{1.77}&\ve{45.14}{1.70}&\ve{29.80}{2.35}&{}\\
\SetCell[r=2]{l,m} MISA     \cite{MISA_Kang_ICLR2025}          &\ascore&\ve{72.60}{0.51}&\ve{79.58}{1.01}&\ve{79.76}{0.76}&\ve{82.86}{0.28}&\ve{67.06}{0.58} &\ve{69.11}{0.73} &\ve{65.26}{0.35} &\ve{77.76}{1.47} &\ve{78.53}{1.48} &\ve{74.46}{1.97} &\ve{67.96}{1.54}&\ve{69.19}{1.46}&\ve{59.63}{0.27}&\SetCell[r=2]{c,m} 5.63\\
{}                                                             &\fscore&\ve{70.94}{0.63}&\ve{78.86}{1.05}&\ve{79.31}{0.85}&\ve{82.44}{0.38}&\ve{66.51}{0.49} &\ve{68.38}{0.62} &\ve{64.27}{0.34} &\ve{75.97}{1.65} &\ve{76.65}{1.68} &\ve{72.30}{2.49} &\ve{64.38}{1.96}&\ve{65.90}{2.01}&\ve{56.34}{0.58}&{}\\
\SetCell[r=2]{l,m} \shortstack[l]{SinglePrompt\\\cite{SinglePrompt_Park_CVPRF2026}} &\ascore&\ve{80.36}{0.37}&\VE{86.27}{0.51}&\VE{86.21}{0.37}&\ve{86.39}{0.34}&\VE{78.25}{0.51} &\VE{79.31}{0.61} &\VE{76.31}{0.20} &\ve{81.59}{0.93} &\ve{79.01}{0.86} &\ve{74.56}{2.49} &\ve{79.06}{0.49}&\ve{81.20}{0.81}&\ve{76.21}{0.47}&\SetCell[r=2]{c,m} 5.85\\
{}                                                             &\fscore&\ve{79.52}{0.39}&\VE{85.89}{0.53}&\VE{85.88}{0.42}&\ve{86.22}{0.38}&\VE{77.19}{0.35} &\VE{78.26}{0.47} &\VE{75.01}{0.24} &\ve{80.79}{1.02} &\ve{78.15}{0.92} &\ve{73.62}{2.86} &\ve{77.70}{0.49}&\ve{80.23}{0.88}&\ve{75.35}{0.52}&{}\\
\SetCell[r=2]{l,m} GACL     \cite{GACL_Zhuang_NeurIPS2024}     &\ascore&\ve{81.84}{0.42}&\ve{73.67}{1.65}&\ve{71.26}{2.04}&\Ve{88.18}{0.00}&\ve{70.91}{0.66} &\ve{73.76}{0.87} &\Ve{72.13}{0.25} &\ve{85.65}{1.60} &\ve{86.79}{1.37} &\Ve{95.82}{0.00} &\ve{86.16}{0.87}&\Ve{89.10}{0.59}&\VE{88.68}{0.00}&\SetCell[r=2]{c,m} 3.41\\
{}                                                             &\fscore&\ve{80.53}{0.58}&\ve{69.60}{2.12}&\ve{67.05}{2.62}&\Ve{88.05}{0.00}&\ve{70.45}{0.56} &\ve{73.59}{0.76} &\Ve{72.32}{0.28} &\ve{83.29}{2.13} &\ve{84.73}{1.75} &\Ve{95.81}{0.00} &\ve{84.65}{1.16}&\Ve{88.30}{0.78}&\VE{88.53}{0.00}&{}\\
\SetCell[r=2]{l,m} SLDA     \cite{SLDA_Hayes_CVPR2020W}        &\ascore&\Ve{83.12}{0.25}&\ve{82.36}{0.90}&\ve{82.16}{0.85}&\ve{86.53}{0.01}&\ve{69.67}{0.61} &\ve{71.08}{0.88} &\ve{67.91}{0.19} &\Ve{89.92}{0.81} &\Ve{90.03}{0.79} &\ve{94.75}{0.01} &\Ve{86.97}{0.57}&\ve{88.79}{0.63}&\ve{87.23}{0.03}&\SetCell[r=2]{c,m} 3.75\\
{}                                                             &\fscore&\Ve{83.55}{0.33}&\ve{81.90}{1.07}&\ve{82.04}{0.90}&\ve{86.78}{0.01}&\Ve{72.23}{0.45} &\ve{73.57}{0.70} &\ve{70.92}{0.21} &\Ve{89.33}{0.95} &\Ve{89.52}{0.83} &\ve{94.82}{0.01} &\Ve{86.03}{0.74}&\ve{88.22}{0.78}&\ve{87.16}{0.03}&{}\\
\midrule
\SetCell[r=2]{l,m} G-AIR                                         &\ascore&\VE{85.44}{0.23}&\Ve{85.78}{0.63}&\Ve{85.92}{0.56}&\VE{88.19}{0.00}&\Ve{73.32}{0.72} &\Ve{74.78}{0.92} &\ve{71.85}{0.31} &\VE{91.75}{0.74} &\VE{91.72}{0.71} &\VE{95.85}{0.00} &\VE{87.93}{0.56}&\VE{89.66}{0.59}&\Ve{88.51}{0.00}&\SetCell[r=2]{c,m} 2.85\\
{}                                                             &\fscore&\VE{84.81}{0.28}&\Ve{85.15}{0.78}&\Ve{85.38}{0.69}&\VE{88.06}{0.00}&\ve{72.11}{0.62} &\Ve{73.77}{0.81} &\ve{70.71}{0.28} &\VE{91.18}{0.88} &\VE{91.13}{0.83} &\VE{95.84}{0.00} &\VE{86.99}{0.71}&\VE{89.09}{0.76}&\Ve{88.37}{0.00}&{}\\
}
    \newcommand{\ARMExtensionRows}{ACIL w/o ARM & \armmean{73.28}{0.71} & \armmean{62.00}{0.44} & \armmean{83.99}{0.20} & \armmean{62.00}{0.44} & \armmean{69.66}{1.45} & \armmean{62.00}{0.44} \\
ACIL w/ ARM & \armimprovement{77.21}{0.74}{3.93}{0.09} & \armimprovement{72.60}{0.57}{10.60}{0.31} & \armimprovement{87.93}{0.14}{3.94}{0.13} & \armimprovement{72.60}{0.57}{10.60}{0.31} & \armimprovement{77.13}{1.64}{7.47}{0.23} & \armimprovement{72.60}{0.57}{10.60}{0.31} \\
\midrule
DS-AL w/o ARM & \armmean{73.90}{0.75} & \armmean{62.66}{0.46} & \armmean{84.70}{0.21} & \armmean{64.41}{0.52} & \armmean{70.97}{1.51} & \armmean{63.63}{0.56} \\
DS-AL w/ ARM & \armimprovement{75.61}{0.77}{1.71}{0.04} & \armimprovement{69.32}{0.57}{6.66}{0.26} & \armimprovement{86.73}{0.16}{2.03}{0.09} & \armimprovement{69.88}{0.56}{5.47}{0.21} & \armimprovement{74.58}{1.65}{3.61}{0.19} & \armimprovement{69.60}{0.59}{5.97}{0.20} \\
}

    \subsection{Scenario 1: Long-Tailed CIL (LT-CIL)}\label{sec:Exp_LT-CIL}
        \par We compare AIR with existing methods on CIFAR-100 \cite{CIFAR_Krizhevsky_2009} and ImageNet-R \cite{ImageNet-R_Hendrycks_ICCV2021} under LT-CIL.

        \subsubsection{The LT-CIL Setting}
            \par We follow \textcite{DAP_Hong_IJCAI2024} to construct long-tailed versions of CIFAR-100 and ImageNet-R. The imbalance ratio $\rho$ is the ratio between the most and least frequent training classes:
            \begin{equation}
                \rho = \frac{\max~\{N_{1:K}^{(y)}\}_{y=0}^{C_K-1}}{\min~\{N_{1:K}^{(y)}\}_{y=0}^{C_K-1}}.
            \end{equation}
            We use $\rho=500$ for CIFAR-100. For ImageNet-R, we use a minimum of three training samples per class in each reported seed. This yields $\rho=70,88.67,96,87,88.33,88$ for seeds 42--47, respectively, with mean $\bar{\rho}=86.33$. ImageNet-R uses data-seed-specific random 80\% training and 20\% test splits. Each data seed randomizes class-to-frequency assignments and within-phase sample order.

            \par Each dataset is divided into 20 phases, introducing 5 CIFAR-100 classes or 10 ImageNet-R classes per phase. Ascending order presents instance-rare classes before instance-rich classes, descending order reverses this sequence, and shuffled order randomly permutes the class groups for each data seed.

        \subsubsection{Evaluation Metrics}
            \par We report \aavg{}, the mean test accuracy over all 20 phase-end evaluations, and \alst{}, the test accuracy after the final phase. For each data seed, Accuracy Score is the unweighted mean of \aavg{} and \alst{} over both datasets and all three orders. The table reports the mean and standard error of these six seed-level Scores. Macro F1 uses the same construction with \favg{} and \flst{} and is reported in \Cref{tab:LT-CIL-F1}. The metrics are formally defined in \Cref{sec:metrics}.

        \subsubsection{Implementation Details}
            \par We use a pre-trained ViT-B/16 \cite{ViT_Dosovitskiy_ICLR2021} as the backbone. AIR uses a frozen ReLU random projection of dimension $f=8192$ and one pass over each phase. For all three orders, $\gamma=20{,}000$ on CIFAR-100 and $\gamma=2{,}000$ on ImageNet-R. The shared training batch size is 32. Each experiment is repeated with six different data seeds (42--47). Standard errors and tests in this paper therefore quantify data-stream variability, not variation across independent model initializations.

            \par Hyper-parameters are selected using training data only. For every method, protocol, and dataset, we reserve the same fixed 10\% of the training data as the validation set and impose the same upper budget of 20 completed candidate configurations. A paper or official configuration counts as one candidate, and the search begins with a small round-number grid.

            \par The reported Time values use CUDA-event timing around method training batches and exclude model loading, data loading and decoding, evaluation, and job-setup time. All the timing records use one NVIDIA A800-SXM4-80GB, PyTorch 2.13.0 with CUDA 12.6, Python 3.14. Data-loader worker counts are method- and dataset-specific and lie outside this timer.

        \subsubsection{Result Analysis}
            \par In \Cref{tab:LT-CIL}, AIR obtains the highest overall Score among the displayed methods (\mbox{$74.56\pm0.31$}), compared with \mbox{$74.32\pm0.38$} for the strongest displayed exemplar-based method (two-stage PODNet \cite{PODNet_Douillard_ECCV2020, LTCIL_Liu_ECCV2022}) and \mbox{$71.35\pm0.36$} for the next exemplar-free method (APART \cite{APART_Qi_ML2025}).

            \par We use one-sided paired $t$-tests on six paired seed-level Scores to assess whether AIR's mean Score exceeds each baseline's, with Holm correction across all 28 baseline comparisons in \Cref{tab:LT-CIL}. In mean aggregate accuracy, AIR significantly outperforms APART, the strongest exemplar-free baseline, by 3.21\% (unadjusted two-sided 95\% paired-$t$ confidence interval [2.74\%, 3.69\%]; Holm-adjusted $p<0.001$). Besides, AIR leads all displayed methods on CIFAR-100 (81.17\%) when averaging \aavg{} and \alst{} across the three orders and ranks second among exemplar-free methods on ImageNet-R (67.95\%, versus 70.07\% for APART). Its recorded learning Time is 3.00~s/1,000 unique samples. The difference of 0.24\% from PODNet/2s is not significant (95\% CI [$-0.26\%$, 0.74\%]; adjusted $p=0.414$).

            \par In \Cref{tab:LT-CIL-F1}, AIR achieves the highest aggregate macro F1-score among exemplar-free methods (\mbox{$73.39\pm0.35$}), compared with \mbox{$71.25\pm0.42$} for the strongest baseline in this group, APER. The observed mean paired gain is 2.14\% (unadjusted two-sided 95\% paired-$t$ CI [1.23\%, 3.05\%]). A one-sided paired $t$-test over six data seeds supports a positive mean difference (Holm-adjusted $p=0.0054$), demonstrating AIR's effectiveness in mitigating class imbalance.

        \begin{table*}[t]
            \centering
            \scriptsize
            \caption{LT-CIL accuracy (\%). The first block uses exemplars, with fixed-memory methods before growing-memory methods: fixed memories retain 500 samples on CIFAR-100 or 1,000 on ImageNet-R, and growing memories retain 5/class. The second block and AIR are exemplar-free. \textbf{Bold}/\underline{underlined} mark the best/second-best exemplar-free results. Values are mean$\pm$standard error over 6 runs. Time is the mean training time in s/1,000 samples across datasets and orders (lower is better).}\label{tab:LT-CIL}
            \begin{tblr}[expand=DGR       \cite{DGR_He_CVPR2024}                             &\ve{61.01}{0.45}&\ve{47.49}{0.77}&\ve{57.85}{0.24}&\ve{85.67}{0.27} &\ve{65.07}{0.44} &\ve{67.67}{3.45}&\ve{60.09}{1.66}&\ve{44.14}{1.54}&\ve{54.35}{0.77}&\ve{72.12}{0.35}&\ve{56.55}{0.32}&\ve{63.60}{1.31}&\ve{57.50}{0.39}&9.19\\
MVP-R     \cite{MVP_Moon_ICCV2023}                           &\ve{61.39}{0.66}&\ve{64.49}{0.57}&\ve{54.57}{0.59}&\ve{83.62}{0.25} &\ve{70.99}{0.68} &\ve{67.60}{2.88}&\ve{59.39}{2.88}&\ve{51.16}{0.58}&\ve{46.42}{0.50}&\ve{69.46}{0.22}&\ve{58.96}{0.36}&\ve{59.59}{1.33}&\ve{50.42}{0.70}&6.84\\
CLIB      \cite{CLIB_Koh_ICLR2022}                           &\ve{62.48}{0.55}&\ve{62.86}{1.08}&\ve{62.36}{0.79}&\ve{84.15}{0.30} &\ve{70.03}{0.26} &\ve{71.63}{2.41}&\ve{65.69}{1.50}&\ve{54.84}{0.59}&\ve{47.56}{0.40}&\ve{66.17}{0.29}&\ve{53.51}{0.34}&\ve{60.34}{1.00}&\ve{50.67}{0.64}&0.51\\
PRS       \cite{PRS_Kim_ECCV2020}                            &\ve{64.21}{0.61}&\ve{49.02}{0.67}&\ve{66.42}{0.77}&\ve{81.29}{0.59} &\ve{64.45}{0.60} &\ve{70.46}{2.53}&\ve{69.81}{1.97}&\ve{46.39}{0.68}&\ve{56.60}{1.72}&\ve{73.66}{0.39}&\ve{66.45}{0.50}&\ve{64.60}{1.30}&\ve{61.34}{0.89}&7.20\\
MEMO      \cite{MEMO_Zhou_ICLR2023}                          &\ve{65.03}{0.47}&\ve{61.29}{0.62}&\ve{66.64}{0.96}&\ve{88.09}{0.27} &\ve{76.37}{0.36} &\ve{74.76}{2.13}&\ve{71.41}{0.89}&\ve{49.89}{0.74}&\ve{50.35}{0.24}&\ve{67.60}{0.30}&\ve{59.59}{0.40}&\ve{59.66}{1.11}&\ve{54.64}{0.34}&4.59\\
MLG       \cite{MLG_Liang_PR2025}                            &\ve{65.84}{1.93}&\ve{66.33}{1.12}&\ve{69.94}{0.64}&\ve{67.52}{10.23}&\ve{51.68}{10.54}&\ve{76.00}{1.65}&\ve{69.87}{0.62}&\ve{54.70}{0.41}&\ve{56.52}{0.41}&\ve{77.06}{0.37}&\ve{62.90}{0.38}&\ve{71.22}{0.91}&\ve{66.36}{0.24}&29.14\\
BiC       \cite{BiC_Wu_CVPR2019}                             &\ve{69.98}{0.33}&\ve{69.60}{1.55}&\ve{73.13}{0.69}&\ve{88.05}{0.25} &\ve{72.65}{0.36} &\ve{74.54}{2.04}&\ve{70.75}{0.49}&\ve{55.56}{0.84}&\ve{64.32}{0.45}&\ve{76.35}{0.22}&\ve{63.20}{0.31}&\ve{67.41}{1.36}&\ve{64.17}{0.45}&7.43\\
ER        \cite{ER_Rolnick_NeurIPS2019}                      &\ve{70.18}{0.49}&\ve{65.35}{0.88}&\ve{63.03}{0.63}&\ve{89.42}{0.31} &\ve{77.91}{0.51} &\ve{73.17}{1.96}&\ve{67.32}{1.20}&\ve{59.56}{0.60}&\ve{60.23}{0.40}&\ve{79.44}{0.21}&\ve{71.02}{0.38}&\ve{70.72}{1.24}&\ve{64.94}{0.53}&6.96\\
DRC       \cite{DRC_Chen_ICCV2023}                           &\ve{62.78}{0.88}&\ve{47.60}{0.83}&\ve{45.82}{1.14}&\ve{83.08}{0.47} &\ve{68.39}{0.91} &\ve{65.71}{4.15}&\ve{50.63}{4.43}&\ve{50.87}{0.93}&\ve{64.38}{0.59}&\ve{76.52}{0.19}&\ve{66.28}{0.18}&\ve{67.60}{1.45}&\ve{66.44}{0.44}&7.02\\
LUCIR     \cite{LUCIR_Hou_CVPR2019}                          &\ve{69.88}{1.51}&\ve{69.21}{0.73}&\ve{73.77}{0.57}&\ve{90.16}{0.32} &\ve{81.73}{0.73} &\ve{64.87}{7.37}&\ve{60.11}{8.03}&\ve{57.64}{0.32}&\ve{65.02}{0.29}&\ve{78.76}{0.35}&\ve{72.85}{0.19}&\ve{64.33}{2.22}&\ve{60.06}{2.45}&9.11\\
UniCIL    \cite{UniCIL_Yang_arXiv2023}                       &\ve{74.20}{0.45}&\ve{73.31}{0.74}&\ve{75.56}{0.68}&\ve{86.88}{0.36} &\ve{74.26}{0.36} &\ve{74.56}{2.68}&\ve{75.84}{0.85}&\ve{70.62}{0.78}&\ve{62.47}{0.48}&\ve{80.64}{0.19}&\ve{70.28}{0.40}&\ve{74.58}{1.77}&\ve{71.44}{0.61}&9.06\\
LUCIR/2s  \cite{LUCIR_Hou_CVPR2019,LTCIL_Liu_ECCV2022}       &\ve{74.24}{0.54}&\ve{72.12}{0.74}&\ve{73.93}{0.60}&\ve{89.86}{0.39} &\ve{79.83}{0.78} &\ve{80.92}{1.93}&\ve{77.02}{1.26}&\ve{59.07}{0.54}&\ve{64.53}{0.30}&\ve{79.07}{0.15}&\ve{71.97}{0.12}&\ve{73.10}{1.39}&\ve{69.43}{0.74}&9.09\\
PODNet/2s \cite{PODNet_Douillard_ECCV2020,LTCIL_Liu_ECCV2022}&\ve{74.32}{0.38}&\ve{76.82}{0.66}&\ve{74.14}{0.47}&\ve{88.18}{0.47} &\ve{79.73}{0.81} &\ve{81.64}{1.48}&\ve{77.92}{0.71}&\ve{61.71}{0.79}&\ve{64.31}{0.41}&\ve{77.31}{0.32}&\ve{70.50}{0.22}&\ve{71.05}{1.13}&\ve{68.51}{0.53}&9.19\\
\midrule
LAE       \cite{LAE_Gao_ICCV2023}                            &\ve{59.28}{0.53}&\ve{66.68}{1.10}&\ve{64.38}{1.30}&\ve{79.73}{1.64} &\ve{58.40}{2.37} &\ve{72.00}{1.50}&\ve{63.91}{1.61}&\ve{45.50}{0.55}&\ve{49.54}{0.38}&\ve{65.90}{0.25}&\ve{43.99}{0.38}&\ve{55.46}{0.92}&\ve{45.88}{0.53}&3.83\\
DAP24     \cite{DAP_Hong_IJCAI2024}                          &\ve{61.97}{0.44}&\ve{66.71}{0.79}&\ve{57.56}{0.48}&\ve{80.86}{0.45} &\ve{65.64}{0.49} &\ve{71.07}{1.87}&\ve{64.58}{0.99}&\ve{52.80}{0.58}&\ve{48.00}{0.40}&\ve{67.80}{0.23}&\ve{54.44}{0.40}&\ve{61.36}{0.95}&\ve{52.80}{0.58}&5.21\\
GACL      \cite{GACL_Zhuang_NeurIPS2024}                     &\ve{64.51}{0.31}&\ve{74.01}{0.75}&\ve{63.16}{0.47}&\ve{84.08}{0.22} &\ve{63.16}{0.47} &\ve{70.66}{1.46}&\ve{63.16}{0.47}&\ve{60.20}{0.65}&\ve{53.80}{0.53}&\ve{72.39}{0.30}&\ve{53.80}{0.53}&\ve{61.92}{1.24}&\ve{53.80}{0.53}&2.56\\
DS-AL     \cite{DS-AL_Zhuang_AAAI2024}                       &\ve{64.90}{0.33}&\ve{73.90}{0.75}&\ve{62.66}{0.46}&\ve{84.70}{0.21} &\ve{64.41}{0.52} &\ve{70.97}{1.51}&\ve{63.63}{0.56}&\ve{59.94}{0.66}&\ve{52.98}{0.52}&\ve{73.11}{0.29}&\ve{55.38}{0.51}&\ve{62.74}{1.20}&\ve{54.43}{0.53}&2.79\\
SLDA      \cite{SLDA_Hayes_CVPR2020W}                        &\ve{67.57}{0.39}&\ve{64.67}{1.02}&\ve{71.02}{0.51}&\ve{87.10}{0.20} &\ve{71.03}{0.51} &\ve{73.59}{2.12}&\ve{71.02}{0.52}&\ve{56.42}{0.81}&\ve{60.06}{0.33}&\ve{71.76}{0.31}&\ve{60.07}{0.33}&\ve{64.06}{1.29}&\ve{60.07}{0.33}&2.47\\
RanPAC    \cite{RanPAC_McDonnell_NeurIPS2023}                &\ve{67.85}{0.70}&\ve{73.92}{0.71}&\ve{63.11}{0.52}&\ve{86.39}{0.30} &\ve{68.36}{0.50} &\ve{72.34}{2.09}&\ve{64.87}{1.28}&\ve{61.14}{0.70}&\ve{55.02}{0.50}&\Ve{77.08}{0.16}&\ve{63.92}{0.30}&\ve{68.00}{2.58}&\ve{60.00}{2.31}&2.22\\
MOS       \cite{MOS_Sun_AAAI2025}                            &\ve{68.03}{0.41}&\Ve{77.54}{0.74}&\ve{74.54}{0.54}&\ve{85.27}{0.55} &\ve{73.26}{0.84} &\Ve{79.98}{1.54}&\ve{74.54}{0.60}&\ve{57.65}{0.61}&\ve{51.59}{0.27}&\ve{68.23}{0.34}&\ve{56.69}{0.45}&\ve{63.45}{0.79}&\ve{53.56}{0.67}&4.74\\
CODA-P    \cite{CODA-Prompt_Smith_CVPR2023}                  &\ve{68.85}{0.46}&\ve{71.20}{1.04}&\ve{67.89}{0.68}&\ve{85.00}{0.44} &\ve{70.32}{0.55} &\ve{75.40}{1.53}&\ve{69.45}{0.76}&\ve{57.86}{0.70}&\ve{58.26}{0.30}&\ve{75.06}{0.38}&\ve{63.83}{0.31}&\ve{69.15}{0.80}&\ve{62.81}{0.36}&7.13\\
FeCAM     \cite{FeCAM_Goswami_NeurIPS2023}                   &\ve{68.98}{0.49}&\ve{74.56}{0.81}&\ve{70.17}{0.61}&\ve{86.44}{0.28} &\ve{71.31}{0.59} &\ve{75.75}{1.83}&\ve{70.54}{0.68}&\ve{55.34}{0.84}&\ve{54.54}{0.47}&\ve{76.18}{0.73}&\ve{64.48}{0.99}&\ve{67.34}{2.11}&\ve{61.05}{1.77}&1.87\\
LwF       \cite{LwF_Li_TPAMI2018}                            &\ve{69.13}{0.53}&\ve{68.94}{0.84}&\ve{67.67}{0.76}&\ve{84.95}{0.64} &\ve{69.58}{1.02} &\ve{76.81}{1.73}&\ve{70.30}{1.03}&\ve{58.52}{0.74}&\ve{61.11}{0.29}&\ve{75.78}{0.22}&\ve{63.40}{0.13}&\ve{69.01}{1.21}&\ve{63.47}{0.41}&8.96\\
SimpleCIL \cite{SimpleCIL_Zhou_IJCV2025}                     &\ve{69.19}{0.37}&\Ve{77.54}{0.74}&\ve{74.57}{0.53}&\ve{86.41}{0.19} &\ve{74.57}{0.53} &\ve{79.69}{1.51}&\ve{74.57}{0.53}&\ve{59.27}{0.60}&\ve{57.18}{0.25}&\ve{67.95}{0.30}&\ve{57.18}{0.25}&\ve{64.12}{0.59}&\ve{57.18}{0.25}&2.44\\
EWC       \cite{EWC_Kirkpatrick_PNAS2017}                    &\ve{69.21}{0.48}&\ve{68.91}{0.72}&\ve{68.58}{0.49}&\ve{85.51}{0.50} &\ve{70.38}{0.73} &\ve{75.88}{1.76}&\ve{69.65}{0.76}&\ve{58.89}{0.48}&\ve{61.46}{0.13}&\ve{75.95}{0.28}&\ve{63.32}{0.35}&\ve{68.58}{1.25}&\ve{63.41}{0.27}&7.71\\
AEF-OCL   \cite{AEF-OCL_Zhuang_TVT2025}                      &\ve{69.85}{0.32}&\ve{76.91}{0.72}&\Ve{74.66}{0.47}&\Ve{87.59}{0.11} &\ve{72.79}{0.44} &\ve{77.60}{1.86}&\ve{73.53}{0.32}&\ve{64.15}{0.64}&\ve{60.39}{0.32}&\ve{70.83}{0.30}&\ve{57.08}{0.38}&\ve{66.06}{0.74}&\ve{56.66}{0.41}&2.63\\
APER      \cite{SimpleCIL_Zhou_IJCV2025}                     &\ve{70.68}{0.43}&\Ve{77.54}{0.73}&\ve{74.56}{0.53}&\ve{86.85}{0.26} &\Ve{75.25}{0.53} &\ve{79.87}{1.52}&\Ve{74.79}{0.57}&\ve{59.27}{0.60}&\ve{57.18}{0.25}&\ve{72.30}{0.47}&\ve{62.11}{0.38}&\ve{67.69}{1.06}&\ve{60.76}{1.04}&1.86\\
APART     \cite{APART_Qi_ML2025}                             &\Ve{71.35}{0.36}&\ve{71.84}{1.12}&\ve{69.97}{0.36}&\ve{84.07}{0.85} &\ve{63.72}{1.69} &\ve{76.45}{1.77}&\ve{69.69}{0.97}&\Ve{64.28}{0.54}&\VE{65.76}{0.33}&\VE{79.82}{0.19}&\VE{67.66}{0.29}&\VE{74.44}{0.94}&\VE{68.44}{0.27}&11.20\\
\midrule
AIR                                                          &\VE{74.56}{0.31}&\VE{79.01}{0.53}&\VE{78.65}{0.37}&\VE{89.09}{0.15} &\VE{78.65}{0.37} &\VE{82.96}{1.25}&\VE{78.65}{0.37}&\VE{66.12}{0.70}&\Ve{64.87}{0.34}&\ve{75.63}{0.33}&\Ve{64.87}{0.34}&\Ve{71.31}{0.68}&\Ve{64.87}{0.34}&3.00\\
]{
                width=\linewidth,colspec={Q[l,m]Q[c,m]*{12}{X[c,m]}Q[c,m]},stretch=1.0,colsep=1pt,rowsep=0.5pt,
                cell{1}{1,2,15}={r=3,c=1}{c,m}, cell{1}{3,9}={r=1,c=6}{c,m},
                cell{2}{3,5,7,9,11,13}={r=1,c=2}{c,m}
            }
                \toprule
                Method & Score &                        CIFAR-100 (LT) &&&&& &                   ImageNet-R (LT) &&&&& & Time \\ \cmidrule[lr]{3-8}\cmidrule[lr]{9-14}
                       &       & Ascending & & Descending & & Shuffled & & Ascending & & Descending & & Shuffled & & \\ \cmidrule[lr]{3-4}\cmidrule[lr]{5-6}\cmidrule[lr]{7-8}\cmidrule[lr]{9-10}\cmidrule[lr]{11-12}\cmidrule[lr]{13-14}
                       &       & \aavg & \alst & \aavg & \alst & \aavg & \alst & \aavg & \alst & \aavg & \alst & \aavg & \alst & \\
                \midrule
                
                \bottomrule
            \end{tblr}
        \end{table*}

        \begin{table*}[t]
            \centering
            \scriptsize
            \caption{LT-CIL macro F1 (\%). The experimental settings, layout, methods, memory grouping, seeds, and timing convention follow \Cref{tab:LT-CIL}. Score is the unweighted mean of \favg{} and \flst{} over both datasets and all three orders. \textbf{Bold}/\underline{underlined} mark the best/second-best exemplar-free results; values are mean$\pm$standard error over six runs.}\label{tab:LT-CIL-F1}
            \begin{tblr}[expand=DGR       \cite{DGR_He_CVPR2024}                             &\ve{59.41}{0.47}&\ve{44.42}{0.90}&\ve{55.22}{0.27}&\ve{85.39}{0.27} &\ve{62.13}{0.49} &\ve{65.24}{3.64}&\ve{57.63}{1.45}&\ve{42.68}{1.37}&\ve{52.31}{0.78}&\ve{71.66}{0.42}&\ve{56.38}{0.24}&\ve{63.00}{1.44}&\ve{56.81}{0.41}&9.19\\
MVP-R     \cite{MVP_Moon_ICCV2023}                           &\ve{61.08}{0.74}&\ve{62.66}{0.76}&\ve{51.87}{0.47}&\ve{83.02}{0.25} &\ve{67.98}{0.70} &\ve{64.74}{3.35}&\ve{57.57}{2.91}&\ve{53.49}{0.52}&\ve{47.10}{0.56}&\ve{71.13}{0.27}&\ve{59.36}{0.37}&\ve{61.12}{1.24}&\ve{52.94}{0.77}&6.84\\
CLIB      \cite{CLIB_Koh_ICLR2022}                           &\ve{63.91}{0.63}&\ve{61.44}{1.30}&\ve{64.25}{0.71}&\ve{84.02}{0.30} &\ve{67.87}{0.35} &\ve{69.49}{2.96}&\ve{65.64}{1.51}&\ve{58.17}{0.58}&\ve{54.47}{0.52}&\ve{68.17}{0.36}&\ve{55.13}{0.29}&\ve{62.63}{0.92}&\ve{55.59}{0.55}&0.51\\
PRS       \cite{PRS_Kim_ECCV2020}                            &\ve{64.04}{0.65}&\ve{47.39}{0.77}&\ve{66.78}{0.66}&\ve{79.91}{0.63} &\ve{59.03}{0.79} &\ve{68.05}{2.85}&\ve{69.37}{1.98}&\ve{49.75}{0.71}&\ve{61.61}{1.39}&\ve{73.86}{0.38}&\ve{62.43}{0.26}&\ve{66.66}{1.17}&\ve{63.63}{0.67}&7.20\\
MLG       \cite{MLG_Liang_PR2025}                            &\ve{64.27}{2.03}&\ve{65.32}{1.28}&\ve{67.98}{0.85}&\ve{65.65}{10.79}&\ve{47.92}{11.34}&\ve{74.00}{1.82}&\ve{67.58}{0.75}&\ve{55.86}{0.53}&\ve{55.97}{0.49}&\ve{75.87}{0.45}&\ve{60.47}{0.38}&\ve{70.05}{0.84}&\ve{64.51}{0.33}&29.14\\
MEMO      \cite{MEMO_Zhou_ICLR2023}                          &\ve{65.48}{0.56}&\ve{58.72}{0.90}&\ve{66.19}{0.96}&\ve{88.01}{0.27} &\ve{74.91}{0.48} &\ve{72.49}{2.58}&\ve{70.10}{1.02}&\ve{52.77}{0.78}&\ve{55.29}{0.28}&\ve{68.58}{0.25}&\ve{59.31}{0.41}&\ve{61.88}{1.00}&\ve{57.50}{0.43}&4.59\\
BiC       \cite{BiC_Wu_CVPR2019}                             &\ve{68.47}{0.40}&\ve{66.83}{1.88}&\ve{70.58}{0.93}&\ve{87.85}{0.26} &\ve{69.80}{0.32} &\ve{71.19}{2.59}&\ve{67.57}{0.50}&\ve{56.04}{0.83}&\ve{63.41}{0.52}&\ve{75.88}{0.28}&\ve{62.66}{0.35}&\ve{66.51}{1.36}&\ve{63.38}{0.58}&7.43\\
ER        \cite{ER_Rolnick_NeurIPS2019}                      &\ve{68.88}{0.56}&\ve{62.94}{1.09}&\ve{61.06}{0.66}&\ve{89.12}{0.33} &\ve{75.76}{0.69} &\ve{69.83}{2.40}&\ve{63.63}{1.31}&\ve{60.78}{0.64}&\ve{61.32}{0.39}&\ve{78.73}{0.29}&\ve{69.64}{0.41}&\ve{70.06}{1.03}&\ve{63.68}{0.57}&6.96\\
DRC       \cite{DRC_Chen_ICCV2023}                           &\ve{62.51}{0.74}&\ve{48.38}{0.78}&\ve{48.87}{1.11}&\ve{82.43}{0.50} &\ve{65.21}{1.05} &\ve{65.92}{3.81}&\ve{54.35}{3.70}&\ve{50.56}{0.84}&\ve{64.49}{0.53}&\ve{75.28}{0.33}&\ve{63.46}{0.35}&\ve{66.16}{1.36}&\ve{65.06}{0.33}&7.02\\
LUCIR     \cite{LUCIR_Hou_CVPR2019}                          &\ve{68.03}{1.57}&\ve{64.35}{0.75}&\ve{71.18}{0.60}&\ve{89.95}{0.38} &\ve{80.73}{0.93} &\ve{62.06}{7.60}&\ve{57.35}{8.34}&\ve{54.20}{0.40}&\ve{63.42}{0.37}&\ve{77.68}{0.60}&\ve{71.54}{0.35}&\ve{63.99}{2.20}&\ve{59.91}{2.45}&9.11\\
LUCIR/2s  \cite{LUCIR_Hou_CVPR2019,LTCIL_Liu_ECCV2022}       &\ve{73.27}{0.57}&\ve{68.12}{0.85}&\ve{72.15}{0.64}&\ve{89.77}{0.41} &\ve{78.75}{0.95} &\ve{79.46}{2.23}&\ve{75.59}{1.46}&\ve{57.75}{0.68}&\ve{64.61}{0.35}&\ve{78.70}{0.32}&\ve{71.53}{0.21}&\ve{73.34}{1.20}&\ve{69.43}{0.60}&9.09\\
UniCIL    \cite{UniCIL_Yang_arXiv2023}                       &\ve{73.57}{0.52}&\ve{70.77}{0.92}&\ve{75.40}{0.67}&\ve{86.19}{0.34} &\ve{71.23}{0.59} &\ve{71.21}{3.18}&\ve{74.31}{0.87}&\ve{72.56}{0.61}&\ve{67.87}{0.50}&\ve{79.58}{0.22}&\ve{68.34}{0.42}&\ve{73.82}{1.75}&\ve{71.49}{0.57}&9.06\\
PODNet/2s \cite{PODNet_Douillard_ECCV2020,LTCIL_Liu_ECCV2022}&\ve{73.98}{0.43}&\ve{75.46}{0.85}&\ve{72.91}{0.40}&\ve{88.07}{0.49} &\ve{78.90}{0.94} &\ve{80.34}{1.75}&\ve{76.71}{0.87}&\ve{63.76}{0.71}&\ve{65.03}{0.35}&\ve{77.10}{0.40}&\ve{69.77}{0.11}&\ve{71.52}{0.92}&\ve{68.26}{0.50}&9.19\\
\midrule
LAE       \cite{LAE_Gao_ICCV2023}                            &\ve{57.57}{0.50}&\ve{65.01}{1.20}&\ve{62.16}{1.09}&\ve{79.34}{1.46} &\ve{56.48}{1.96} &\ve{69.92}{1.68}&\ve{61.98}{1.63}&\ve{45.07}{0.63}&\ve{47.30}{0.39}&\ve{65.25}{0.37}&\ve{41.61}{0.52}&\ve{53.34}{0.92}&\ve{43.40}{0.53}&3.83\\
DAP24     \cite{DAP_Hong_IJCAI2024}                          &\ve{61.95}{0.42}&\ve{65.66}{0.84}&\ve{55.23}{0.45}&\ve{80.82}{0.44} &\ve{65.70}{0.67} &\ve{69.60}{2.24}&\ve{63.73}{1.27}&\ve{53.99}{0.60}&\ve{46.76}{0.39}&\ve{68.38}{0.34}&\ve{56.24}{0.29}&\ve{62.63}{0.64}&\ve{54.72}{0.47}&5.21\\
GACL      \cite{GACL_Zhuang_NeurIPS2024}                     &\ve{62.49}{0.35}&\ve{71.65}{0.97}&\ve{59.52}{0.61}&\ve{83.36}{0.21} &\ve{59.52}{0.61} &\ve{67.27}{1.76}&\ve{59.52}{0.61}&\ve{59.72}{0.69}&\ve{52.04}{0.52}&\ve{72.36}{0.43}&\ve{52.04}{0.52}&\ve{60.88}{1.18}&\ve{52.04}{0.52}&2.56\\
DS-AL     \cite{DS-AL_Zhuang_AAAI2024}                       &\ve{62.91}{0.37}&\ve{71.55}{0.97}&\ve{59.01}{0.59}&\ve{84.02}{0.21} &\ve{60.91}{0.67} &\ve{67.65}{1.82}&\ve{60.09}{0.71}&\ve{59.55}{0.72}&\ve{51.15}{0.50}&\ve{73.00}{0.43}&\ve{53.58}{0.55}&\ve{61.82}{1.08}&\ve{52.64}{0.51}&2.79\\
RanPAC    \cite{RanPAC_McDonnell_NeurIPS2023}                &\ve{66.49}{0.67}&\ve{71.53}{0.93}&\ve{59.52}{0.63}&\ve{86.01}{0.27} &\ve{65.87}{0.60} &\ve{69.18}{2.40}&\ve{61.46}{1.43}&\ve{60.94}{0.74}&\ve{53.92}{0.49}&\Ve{77.33}{0.20}&\ve{64.04}{0.39}&\ve{68.12}{2.30}&\ve{59.99}{1.96}&2.22\\
CODA-P    \cite{CODA-Prompt_Smith_CVPR2023}                  &\ve{67.77}{0.52}&\ve{69.65}{1.21}&\ve{65.27}{0.82}&\ve{84.52}{0.46} &\ve{68.15}{0.75} &\ve{73.49}{1.78}&\ve{67.06}{0.95}&\ve{58.91}{0.52}&\ve{57.74}{0.37}&\ve{74.63}{0.50}&\ve{63.36}{0.25}&\ve{68.53}{0.69}&\ve{61.99}{0.49}&7.13\\
LwF       \cite{LwF_Li_TPAMI2018}                            &\ve{68.22}{0.57}&\ve{67.14}{0.76}&\ve{65.21}{0.86}&\ve{84.51}{0.60} &\ve{67.64}{1.10} &\ve{75.15}{1.90}&\ve{68.48}{1.18}&\ve{59.17}{0.72}&\ve{60.40}{0.36}&\ve{75.74}{0.36}&\ve{63.58}{0.24}&\ve{68.62}{1.16}&\ve{62.98}{0.38}&8.96\\
EWC       \cite{EWC_Kirkpatrick_PNAS2017}                    &\ve{68.29}{0.53}&\ve{67.17}{0.91}&\ve{66.11}{0.59}&\ve{85.18}{0.51} &\ve{68.54}{0.88} &\ve{73.97}{2.01}&\ve{67.61}{0.85}&\ve{59.56}{0.52}&\ve{60.96}{0.31}&\ve{75.96}{0.30}&\ve{63.43}{0.41}&\ve{68.08}{1.20}&\ve{62.95}{0.40}&7.71\\
SLDA      \cite{SLDA_Hayes_CVPR2020W}                        &\ve{68.70}{0.39}&\ve{63.90}{1.31}&\ve{69.49}{0.60}&\ve{87.15}{0.15} &\ve{69.51}{0.60} &\ve{71.59}{2.39}&\ve{69.50}{0.61}&\ve{60.67}{0.79}&\VE{63.68}{0.38}&\ve{74.46}{0.45}&\ve{63.69}{0.39}&\ve{67.05}{1.08}&\Ve{63.69}{0.39}&2.47\\
FeCAM     \cite{FeCAM_Goswami_NeurIPS2023}                   &\ve{68.81}{0.48}&\ve{73.14}{1.02}&\ve{68.31}{0.76}&\ve{86.26}{0.25} &\ve{69.53}{0.73} &\ve{73.87}{2.13}&\ve{68.70}{0.83}&\ve{57.75}{0.92}&\ve{56.56}{0.55}&\ve{76.43}{0.56}&\Ve{64.83}{0.95}&\ve{68.31}{1.76}&\ve{62.04}{1.34}&1.87\\
AEF-OCL   \cite{AEF-OCL_Zhuang_TVT2025}                      &\ve{69.04}{0.39}&\ve{74.68}{0.96}&\ve{72.33}{0.65}&\Ve{87.30}{0.11} &\ve{70.37}{0.65} &\ve{75.09}{2.40}&\ve{71.32}{0.55}&\Ve{64.53}{0.61}&\ve{60.91}{0.37}&\ve{71.07}{0.40}&\ve{57.15}{0.54}&\ve{66.50}{0.68}&\ve{57.28}{0.56}&2.63\\
APART     \cite{APART_Qi_ML2025}                             &\ve{69.38}{0.43}&\ve{69.70}{1.30}&\ve{66.77}{0.55}&\ve{83.35}{0.91} &\ve{61.32}{1.56} &\ve{73.97}{2.11}&\ve{66.60}{1.22}&\ve{63.83}{0.60}&\Ve{63.49}{0.51}&\VE{78.80}{0.32}&\VE{65.79}{0.37}&\VE{72.83}{0.96}&\VE{66.07}{0.35}&11.20\\
MOS       \cite{MOS_Sun_AAAI2025}                            &\ve{69.56}{0.47}&\Ve{76.30}{0.91}&\ve{73.34}{0.69}&\ve{85.20}{0.55} &\ve{72.03}{0.98} &\Ve{78.74}{1.74}&\ve{73.28}{0.72}&\ve{61.09}{0.46}&\ve{56.88}{0.19}&\ve{70.89}{0.40}&\ve{61.79}{0.37}&\ve{66.33}{0.62}&\ve{58.87}{0.46}&4.74\\
SimpleCIL \cite{SimpleCIL_Zhou_IJCV2025}                     &\ve{69.98}{0.41}&\Ve{76.30}{0.91}&\Ve{73.37}{0.68}&\ve{86.34}{0.17} &\ve{73.37}{0.68} &\ve{78.46}{1.75}&\ve{73.37}{0.68}&\ve{62.21}{0.44}&\ve{60.00}{0.29}&\ve{70.04}{0.41}&\ve{60.00}{0.29}&\ve{66.30}{0.42}&\ve{60.00}{0.29}&2.44\\
APER      \cite{SimpleCIL_Zhou_IJCV2025}                     &\Ve{71.25}{0.42}&\Ve{76.30}{0.91}&\ve{73.36}{0.67}&\ve{86.78}{0.24} &\Ve{74.07}{0.67} &\ve{78.65}{1.76}&\Ve{73.59}{0.70}&\ve{62.20}{0.44}&\ve{59.99}{0.29}&\ve{73.60}{0.36}&\ve{64.08}{0.37}&\ve{69.38}{0.81}&\ve{62.97}{0.75}&1.86\\
\midrule
AIR                                                          &\VE{73.39}{0.35}&\VE{78.12}{0.57}&\VE{77.17}{0.51}&\VE{88.83}{0.15} &\VE{77.17}{0.51} &\VE{81.63}{1.50}&\VE{77.17}{0.51}&\VE{65.58}{0.65}&\ve{63.34}{0.39}&\ve{74.75}{0.43}&\ve{63.34}{0.39}&\Ve{70.22}{0.67}&\ve{63.34}{0.39}&3.00\\
]{
                width=\linewidth,colspec={Q[l,m]Q[c,m]*{12}{X[c,m]}Q[c,m]},stretch=1.0,colsep=1pt,rowsep=0.5pt,
                cell{1}{1,2,15}={r=3,c=1}{c,m}, cell{1}{3,9}={r=1,c=6}{c,m},
                cell{2}{3,5,7,9,11,13}={r=1,c=2}{c,m}
            }
                \toprule
                Method & Score & CIFAR-100 (LT) &&&&& & ImageNet-R (LT) &&&&& & Time \\ \cmidrule[lr]{3-8}\cmidrule[lr]{9-14}
                       &       & Ascending & & Descending & & Shuffled & & Ascending & & Descending & & Shuffled & & \\ \cmidrule[lr]{3-4}\cmidrule[lr]{5-6}\cmidrule[lr]{7-8}\cmidrule[lr]{9-10}\cmidrule[lr]{11-12}\cmidrule[lr]{13-14}
                       &       & \favg & \flst & \favg & \flst & \favg & \flst & \favg & \flst & \favg & \flst & \favg & \flst & \\
                \midrule
                
                \bottomrule
            \end{tblr}
        \end{table*}

    \subsection{Scenario 2: Generalized CIL (GCIL)}\label{sec:Exp_GCIL}
        \par Following \textcite{MVP_Moon_ICCV2023}, we evaluate G-AIR on CIFAR-100 \cite{CIFAR_Krizhevsky_2009}, ImageNet-R \cite{ImageNet-R_Hendrycks_ICCV2021}, CORe50 \cite{CORe50_Lomonaco_ACRL2017}, and CUB-200-2011 \cite{CUB-200_Caltech2011} under the Si-Blurry GCIL setting.

        \subsubsection{The Si-Blurry Setting}
            \par Si-Blurry partitions classes into phase-disjoint and blurry groups. Samples from a blurry class may be redistributed across phases. We use 5 phases, a disjoint-class ratio $r_{\text{D}}=0.5$, a blurry-sample ratio $r_{\text{B}}=0.1$, and randomized class/group assignments.

        \subsubsection{Evaluation Metrics}
            \par We follow \textcite{MVP_Moon_ICCV2023} to report AUC, Avg, and Last for accuracy and macro F1. AUC is the arithmetic mean of periodic evaluations at batch ends strictly beyond successive 1,000-unique-sample thresholds, with no forced terminal evaluation (\Cref{sec:metrics}). Avg is the mean of the five phase-end evaluations, and Last is the final phase-end result. Accuracy Score and macro F1-score are computed separately as the unweighted mean of their respective AUC, Avg, and Last values over the four datasets.

        \subsubsection{Implementation Details}
            \par We use the same ViT-B/16 backbone as in LT-CIL and a shared stream batch size of 32. G-AIR uses a random projection of dimension $f=8192$; $\gamma$ is 5,000, 1,000, 2,000, and 100 on CIFAR-100 \cite{CIFAR_Krizhevsky_2009}, ImageNet-R \cite{ImageNet-R_Hendrycks_ICCV2021}, CORe50 \cite{CORe50_Lomonaco_ACRL2017}, and CUB-200-2011 \cite{CUB-200_Caltech2011}, respectively. Each experiment is repeated with six different data seeds (42--47).
            \par For Si-Blurry, the fixed holdout is randomly sampled per class across the full training stream (approximately 10\%), retaining each class's first occurrence per phase in training. Tuning evaluates only held-out samples of classes observed so far, without using test data.

        \subsubsection{Result Analysis}
            \par G-AIR achieves the highest aggregate Accuracy Score (\mbox{$85.44\pm0.23$}) and macro F1-score (\mbox{$84.81\pm0.28$}) in \Cref{tab:Results_SiBlurry}, with a recorded learning Time of 2.85~s/1,000 unique samples and no RAM data preloading. It leads all six metrics on CORe50 and the AUC and Avg cells on CUB-200-2011. SinglePrompt (Single-P in figures) \cite{SinglePrompt_Park_CVPRF2026} leads these cells on CIFAR-100 and ImageNet-R, where G-AIR ranks second except for ImageNet-R macro F1 AUC (third).

            \par Compared with GACL \cite{GACL_Zhuang_NeurIPS2024}, G-AIR increases the mean aggregate Accuracy Score from 81.84 to 85.44 and the mean macro F1-score from 80.53 to 84.81. Its gains over the strongest baseline, SLDA, are 2.32\% in accuracy and 1.27\% in macro F1. One-sided paired $t$-tests support positive mean gains (Holm-adjusted $p=1.50\times10^{-7}$ and $2.81\times10^{-5}$, respectively, across 15 comparisons per metric). G-AIR also has higher aggregate scores than every displayed exemplar-based method.

            \begin{table*}[t]
            \centering
            \scriptsize
                \caption{Si-Blurry accuracy (\ascore) and macro F1 (\fscore), in \% (mean$\pm$standard error, 6 runs). The first block retains samples, ordered by increasing Accuracy Score; the rest are exemplar-free. Fixed memories hold 500/1,000/250/1,000 samples in dataset-column order; growing memories hold 5/class. O-LoRA retains four hard samples only for MAS estimation. \textbf{Bold}/\underline{underlined}: best/second-best exemplar-free results. Time is the mean training time in s/1,000 samples across datasets (lower is better).}\label{tab:Results_SiBlurry}
                \begin{tblr}[expand=\SetCell[r=2]{l,m} O-LoRA   \cite{Online-LoRA_Wei_WACV2025}&\ascore&\ve{54.95}{0.98}&\ve{68.15}{2.52}&\ve{68.02}{1.83}&\ve{68.20}{2.69}&\ve{61.45}{0.98} &\ve{60.85}{1.06} &\ve{53.54}{1.70} &\ve{47.06}{2.17} &\ve{50.73}{3.33} &\ve{50.84}{3.78} &\ve{40.54}{0.71}&\ve{47.50}{1.66}&\ve{42.60}{1.07}&\SetCell[r=2]{c,m} 5.10\\
{}                                                             &\fscore&\ve{50.54}{1.14}&\ve{64.44}{3.07}&\ve{65.09}{2.25}&\ve{65.83}{2.98}&\ve{59.57}{1.18} &\ve{59.49}{1.14} &\ve{53.91}{1.50} &\ve{40.23}{2.62} &\ve{43.44}{4.13} &\ve{44.41}{4.56} &\ve{33.30}{0.95}&\ve{40.97}{2.17}&\ve{35.80}{1.43}&{}\\
\SetCell[r=2]{l,m} UniCIL   \cite{UniCIL_Yang_arXiv2023}       &\ascore&\ve{57.39}{4.70}&\ve{78.00}{0.97}&\ve{72.96}{1.45}&\ve{85.38}{0.34}&\ve{37.34}{15.94}&\ve{38.28}{16.41}&\ve{38.83}{16.87}&\ve{49.06}{16.10}&\ve{49.84}{16.21}&\ve{49.93}{16.92}&\ve{62.77}{2.05}&\ve{63.29}{2.97}&\ve{63.02}{4.83}&\SetCell[r=2]{c,m} 7.95\\
{}                                                             &\fscore&\ve{54.42}{4.78}&\ve{75.27}{1.47}&\ve{69.64}{1.60}&\ve{85.35}{0.32}&\ve{35.47}{15.85}&\ve{36.26}{16.20}&\ve{37.91}{16.93}&\ve{46.49}{16.42}&\ve{47.47}{16.76}&\ve{48.78}{17.38}&\ve{56.09}{2.53}&\ve{56.39}{3.43}&\ve{57.90}{5.94}&{}\\
\SetCell[r=2]{l,m} EWC++    \cite{EWC_Kirkpatrick_PNAS2017}    &\ascore&\ve{72.55}{0.48}&\ve{72.44}{1.22}&\ve{70.95}{0.86}&\ve{69.91}{0.79}&\ve{61.96}{0.76} &\ve{61.62}{1.19} &\ve{54.16}{0.58} &\ve{81.28}{0.68} &\ve{81.81}{0.67} &\ve{82.54}{1.29} &\ve{77.49}{1.02}&\ve{81.10}{0.98}&\ve{75.28}{0.51}&\SetCell[r=2]{c,m} 0.50\\
{}                                                             &\fscore&\ve{72.30}{0.48}&\ve{71.53}{1.43}&\ve{70.21}{1.07}&\ve{70.35}{0.75}&\ve{62.74}{0.63} &\ve{62.41}{0.92} &\ve{55.88}{0.52} &\ve{79.92}{0.82} &\ve{80.87}{0.80} &\ve{82.54}{1.23} &\ve{76.12}{1.12}&\ve{80.22}{1.04}&\ve{74.75}{0.74}&{}\\
\SetCell[r=2]{l,m} RM       \cite{RM_Bang_CVPR2021}            &\ascore&\ve{72.73}{0.58}&\ve{73.22}{1.34}&\ve{84.08}{0.46}&\ve{83.16}{0.17}&\ve{64.01}{1.48} &\ve{64.45}{2.10} &\ve{58.46}{1.28} &\ve{81.72}{1.84} &\ve{85.85}{1.00} &\ve{85.32}{1.13} &\ve{42.95}{3.23}&\ve{80.22}{1.61}&\ve{69.38}{1.16}&\SetCell[r=2]{c,m} 5.03\\
{}                                                             &\fscore&\ve{71.59}{0.61}&\ve{70.41}{1.77}&\ve{83.56}{0.52}&\ve{82.87}{0.19}&\ve{63.95}{1.29} &\ve{63.60}{1.91} &\ve{57.72}{1.18} &\ve{79.95}{2.26} &\ve{85.45}{1.12} &\ve{85.15}{1.11} &\ve{37.98}{3.05}&\ve{79.51}{1.70}&\ve{68.90}{1.16}&{}\\
\SetCell[r=2]{l,m} ER       \cite{ER_Rolnick_NeurIPS2019}      &\ascore&\ve{73.26}{0.48}&\ve{73.24}{1.27}&\ve{71.85}{0.97}&\ve{71.02}{1.07}&\ve{63.15}{0.67} &\ve{62.82}{1.22} &\ve{55.63}{0.71} &\ve{81.60}{0.69} &\ve{81.85}{0.57} &\ve{82.22}{1.40} &\ve{78.49}{1.41}&\ve{81.50}{0.91}&\ve{75.68}{0.38}&\SetCell[r=2]{c,m} 0.45\\
{}                                                             &\fscore&\ve{73.02}{0.47}&\ve{72.27}{1.48}&\ve{71.13}{1.21}&\ve{71.55}{1.01}&\ve{63.85}{0.50} &\ve{63.75}{0.97} &\ve{57.63}{0.54} &\ve{80.30}{0.86} &\ve{80.77}{0.76} &\ve{82.16}{1.36} &\ve{77.18}{1.50}&\ve{80.59}{0.95}&\ve{75.12}{0.36}&{}\\
\SetCell[r=2]{l,m} CLIB     \cite{CLIB_Koh_ICLR2022}           &\ascore&\ve{74.13}{0.45}&\ve{74.91}{1.30}&\ve{74.11}{0.90}&\ve{70.70}{0.94}&\ve{62.95}{0.72} &\ve{62.05}{1.16} &\ve{51.97}{0.51} &\ve{81.92}{0.32} &\ve{81.81}{0.65} &\ve{78.75}{1.57} &\ve{84.24}{0.77}&\ve{85.31}{0.86}&\ve{80.85}{0.32}&\SetCell[r=2]{c,m} 0.59\\
{}                                                             &\fscore&\ve{74.41}{0.44}&\ve{74.41}{1.48}&\ve{73.91}{1.00}&\ve{71.21}{0.96}&\ve{64.22}{0.61} &\ve{63.84}{0.80} &\ve{55.69}{0.43} &\ve{81.05}{0.42} &\ve{81.32}{0.69} &\ve{78.91}{1.46} &\ve{83.17}{0.98}&\ve{84.65}{0.99}&\ve{80.54}{0.39}&{}\\
\SetCell[r=2]{l,m} MVP-R    \cite{MVP_Moon_ICCV2023}           &\ascore&\ve{77.84}{0.34}&\ve{78.62}{1.16}&\ve{77.39}{1.00}&\ve{78.21}{0.74}&\ve{69.81}{0.61} &\ve{69.85}{0.85} &\ve{64.12}{0.61} &\ve{84.59}{0.71} &\ve{84.66}{0.89} &\ve{87.29}{0.78} &\ve{79.91}{1.07}&\ve{82.18}{1.15}&\ve{77.44}{0.53}&\SetCell[r=2]{c,m} 5.43\\
{}                                                             &\fscore&\ve{77.34}{0.33}&\ve{77.31}{1.36}&\ve{76.29}{1.21}&\ve{78.05}{0.75}&\ve{69.88}{0.50} &\ve{70.25}{0.63} &\ve{65.57}{0.37} &\ve{83.46}{0.87} &\ve{83.67}{1.14} &\ve{87.28}{0.78} &\ve{78.47}{1.17}&\ve{80.98}{1.34}&\ve{76.83}{0.55}&{}\\
\midrule
\SetCell[r=2]{l,m} Finetune                                    &\ascore&\ve{54.36}{0.78}&\ve{63.92}{2.04}&\ve{62.68}{1.51}&\ve{62.62}{2.24}&\ve{51.99}{0.74} &\ve{52.70}{0.68} &\ve{47.68}{0.75} &\ve{61.72}{2.05} &\ve{61.59}{1.66} &\ve{50.79}{3.79} &\ve{53.95}{1.02}&\ve{49.98}{2.18}&\ve{32.72}{2.91}&\SetCell[r=2]{c,m} 0.68\\
{}                                                             &\fscore&\ve{50.39}{0.84}&\ve{60.65}{2.52}&\ve{60.04}{1.85}&\ve{60.95}{2.44}&\ve{51.46}{0.75} &\ve{52.79}{0.99} &\ve{49.05}{0.51} &\ve{55.08}{2.50} &\ve{56.33}{2.26} &\ve{42.72}{4.53} &\ve{46.96}{1.44}&\ve{43.07}{2.19}&\ve{25.54}{2.53}&{}\\
\SetCell[r=2]{l,m} LwF      \cite{LwF_Li_TPAMI2018}            &\ascore&\ve{54.78}{0.80}&\ve{63.58}{2.13}&\ve{62.11}{1.75}&\ve{62.63}{2.29}&\ve{56.29}{0.82} &\ve{55.34}{0.67} &\ve{46.99}{1.12} &\ve{61.62}{1.99} &\ve{61.34}{1.60} &\ve{50.28}{3.67} &\ve{54.04}{1.09}&\ve{50.04}{2.15}&\ve{33.10}{2.91}&\SetCell[r=2]{c,m} 0.69\\
{}                                                             &\fscore&\ve{50.86}{0.86}&\ve{60.19}{2.65}&\ve{59.36}{2.15}&\ve{61.05}{2.43}&\ve{56.25}{0.80} &\ve{55.58}{0.62} &\ve{48.74}{0.62} &\ve{54.95}{2.44} &\ve{55.97}{2.18} &\ve{42.23}{4.31} &\ve{47.08}{1.47}&\ve{43.08}{2.15}&\ve{25.85}{2.56}&{}\\
\SetCell[r=2]{l,m} Dual-P   \cite{DualPrompt_Wang_ECCV2022}    &\ascore&\ve{58.08}{0.68}&\ve{70.88}{1.69}&\ve{70.73}{1.12}&\ve{68.38}{1.90}&\ve{58.55}{0.73} &\ve{58.03}{0.66} &\ve{51.07}{1.18} &\ve{63.18}{2.11} &\ve{63.27}{1.59} &\ve{52.60}{3.57} &\ve{55.31}{1.11}&\ve{51.12}{2.15}&\ve{33.79}{2.97}&\SetCell[r=2]{c,m} 4.45\\
{}                                                             &\fscore&\ve{54.47}{0.74}&\ve{68.46}{2.03}&\ve{68.87}{1.35}&\ve{66.48}{2.10}&\ve{58.17}{0.78} &\ve{57.97}{0.79} &\ve{52.67}{0.66} &\ve{57.13}{2.66} &\ve{58.60}{2.13} &\ve{45.33}{4.31} &\ve{48.64}{1.46}&\ve{44.50}{2.14}&\ve{26.88}{2.64}&{}\\
\SetCell[r=2]{l,m} MVP      \cite{MVP_Moon_ICCV2023}           &\ascore&\ve{58.41}{0.81}&\ve{69.38}{1.88}&\ve{69.11}{1.40}&\ve{68.42}{2.10}&\ve{60.24}{0.80} &\ve{58.88}{1.04} &\ve{49.28}{1.37} &\ve{63.98}{2.09} &\ve{63.89}{1.66} &\ve{52.90}{3.99} &\ve{55.70}{1.30}&\ve{52.17}{1.75}&\ve{36.97}{2.57}&\SetCell[r=2]{c,m} 5.45\\
{}                                                             &\fscore&\ve{54.55}{0.85}&\ve{66.51}{2.26}&\ve{66.90}{1.67}&\ve{66.73}{2.32}&\ve{59.74}{0.93} &\ve{58.65}{0.70} &\ve{50.18}{0.74} &\ve{57.83}{2.58} &\ve{59.04}{2.13} &\ve{45.28}{4.77} &\ve{48.85}{1.77}&\ve{45.14}{1.70}&\ve{29.80}{2.35}&{}\\
\SetCell[r=2]{l,m} MISA     \cite{MISA_Kang_ICLR2025}          &\ascore&\ve{72.60}{0.51}&\ve{79.58}{1.01}&\ve{79.76}{0.76}&\ve{82.86}{0.28}&\ve{67.06}{0.58} &\ve{69.11}{0.73} &\ve{65.26}{0.35} &\ve{77.76}{1.47} &\ve{78.53}{1.48} &\ve{74.46}{1.97} &\ve{67.96}{1.54}&\ve{69.19}{1.46}&\ve{59.63}{0.27}&\SetCell[r=2]{c,m} 5.63\\
{}                                                             &\fscore&\ve{70.94}{0.63}&\ve{78.86}{1.05}&\ve{79.31}{0.85}&\ve{82.44}{0.38}&\ve{66.51}{0.49} &\ve{68.38}{0.62} &\ve{64.27}{0.34} &\ve{75.97}{1.65} &\ve{76.65}{1.68} &\ve{72.30}{2.49} &\ve{64.38}{1.96}&\ve{65.90}{2.01}&\ve{56.34}{0.58}&{}\\
\SetCell[r=2]{l,m} \shortstack[l]{SinglePrompt\\\cite{SinglePrompt_Park_CVPRF2026}} &\ascore&\ve{80.36}{0.37}&\VE{86.27}{0.51}&\VE{86.21}{0.37}&\ve{86.39}{0.34}&\VE{78.25}{0.51} &\VE{79.31}{0.61} &\VE{76.31}{0.20} &\ve{81.59}{0.93} &\ve{79.01}{0.86} &\ve{74.56}{2.49} &\ve{79.06}{0.49}&\ve{81.20}{0.81}&\ve{76.21}{0.47}&\SetCell[r=2]{c,m} 5.85\\
{}                                                             &\fscore&\ve{79.52}{0.39}&\VE{85.89}{0.53}&\VE{85.88}{0.42}&\ve{86.22}{0.38}&\VE{77.19}{0.35} &\VE{78.26}{0.47} &\VE{75.01}{0.24} &\ve{80.79}{1.02} &\ve{78.15}{0.92} &\ve{73.62}{2.86} &\ve{77.70}{0.49}&\ve{80.23}{0.88}&\ve{75.35}{0.52}&{}\\
\SetCell[r=2]{l,m} GACL     \cite{GACL_Zhuang_NeurIPS2024}     &\ascore&\ve{81.84}{0.42}&\ve{73.67}{1.65}&\ve{71.26}{2.04}&\Ve{88.18}{0.00}&\ve{70.91}{0.66} &\ve{73.76}{0.87} &\Ve{72.13}{0.25} &\ve{85.65}{1.60} &\ve{86.79}{1.37} &\Ve{95.82}{0.00} &\ve{86.16}{0.87}&\Ve{89.10}{0.59}&\VE{88.68}{0.00}&\SetCell[r=2]{c,m} 3.41\\
{}                                                             &\fscore&\ve{80.53}{0.58}&\ve{69.60}{2.12}&\ve{67.05}{2.62}&\Ve{88.05}{0.00}&\ve{70.45}{0.56} &\ve{73.59}{0.76} &\Ve{72.32}{0.28} &\ve{83.29}{2.13} &\ve{84.73}{1.75} &\Ve{95.81}{0.00} &\ve{84.65}{1.16}&\Ve{88.30}{0.78}&\VE{88.53}{0.00}&{}\\
\SetCell[r=2]{l,m} SLDA     \cite{SLDA_Hayes_CVPR2020W}        &\ascore&\Ve{83.12}{0.25}&\ve{82.36}{0.90}&\ve{82.16}{0.85}&\ve{86.53}{0.01}&\ve{69.67}{0.61} &\ve{71.08}{0.88} &\ve{67.91}{0.19} &\Ve{89.92}{0.81} &\Ve{90.03}{0.79} &\ve{94.75}{0.01} &\Ve{86.97}{0.57}&\ve{88.79}{0.63}&\ve{87.23}{0.03}&\SetCell[r=2]{c,m} 3.75\\
{}                                                             &\fscore&\Ve{83.55}{0.33}&\ve{81.90}{1.07}&\ve{82.04}{0.90}&\ve{86.78}{0.01}&\Ve{72.23}{0.45} &\ve{73.57}{0.70} &\ve{70.92}{0.21} &\Ve{89.33}{0.95} &\Ve{89.52}{0.83} &\ve{94.82}{0.01} &\Ve{86.03}{0.74}&\ve{88.22}{0.78}&\ve{87.16}{0.03}&{}\\
\midrule
\SetCell[r=2]{l,m} G-AIR                                         &\ascore&\VE{85.44}{0.23}&\Ve{85.78}{0.63}&\Ve{85.92}{0.56}&\VE{88.19}{0.00}&\Ve{73.32}{0.72} &\Ve{74.78}{0.92} &\ve{71.85}{0.31} &\VE{91.75}{0.74} &\VE{91.72}{0.71} &\VE{95.85}{0.00} &\VE{87.93}{0.56}&\VE{89.66}{0.59}&\Ve{88.51}{0.00}&\SetCell[r=2]{c,m} 2.85\\
{}                                                             &\fscore&\VE{84.81}{0.28}&\Ve{85.15}{0.78}&\Ve{85.38}{0.69}&\VE{88.06}{0.00}&\ve{72.11}{0.62} &\Ve{73.77}{0.81} &\ve{70.71}{0.28} &\VE{91.18}{0.88} &\VE{91.13}{0.83} &\VE{95.84}{0.00} &\VE{86.99}{0.71}&\VE{89.09}{0.76}&\Ve{88.37}{0.00}&{}\\
]{
                    width=\linewidth,colspec={Q[l,m]Q[c,m]Q[c,m]*{12}{X[c,m]}Q[c,m]},stretch=1.0,colsep=1pt,rowsep=0.5pt,
                    cell{1}{1,2,3,16}={r=2,c=1}{c,m},cell{1}{4,7,10,13}={r=1,c=3}{c,m}
                }
                    \toprule
                    Method & Type & Score & CIFAR-100 && & ImageNet-R && & CORe50 && & CUB-200-2011 && & Time \\ \cmidrule[lr]{4-6}\cmidrule[lr]{7-9}\cmidrule[lr]{10-12}\cmidrule[lr]{13-15}
                           &      &       & AUC & Avg & Last & AUC & Avg & Last & AUC & Avg & Last & AUC & Avg & Last & \\ \midrule
                    
                    \bottomrule
                \end{tblr}
            \end{table*}

\subsection{Learning-Progress and Forgetting Analysis}
    \par AIR retains accuracy as the set of classes to distinguish expands (\Cref{fig:catastrophic_forgetting_progress}). On ascending CIFAR-100 LT-CIL, its seen-class accuracy decreases by only 2.68\%, from 81.33\% to 78.65\%, compared with 16.42\%--28.48\% for the plotted baselines. To separate old-class retention from changes in the evaluation class set, we also measure each class's final accuracy minus its accuracy at the end of its first observed phase, averaging over classes introduced before the final phase and then over six seeds. AIR loses 6.35\% on these old classes, versus 15.03\%--31.84\% for the baselines, indicating significantly less forgetting (paired one-sided $t$-tests over six seeds; Holm-adjusted $p<0.001$ within a separate family of five AIR-versus-baseline forgetting comparisons).

    \par On CIFAR-100 Si-Blurry, G-AIR's periodic accuracy rises from 81.66\% to 88.14\%. Its old-class accuracy improves by 5.56\%, exceeding MISA (4.65\%), MVP ($-7.59\%$), and SinglePrompt ($-0.63\%$), although GACL and SLDA gain more. For G-AIR, recurring old classes gain 11.93\%, while non-recurring old classes lose 2.77\%. This positive class-wise backward transfer is consistent with effective use of newly arriving samples from old classes, rather than transfer from new classes alone. AIR nevertheless trails the baselines early in LT-CIL and SinglePrompt early in Si-Blurry, consistent with limited feature adaptation under a frozen backbone. Jointly adapting the backbone and classifier may improve this early performance, provided that the sufficient statistics remain consistent with the updated features.

    \begin{figure}[ht]
        \centering
        \includegraphics[width=\linewidth]{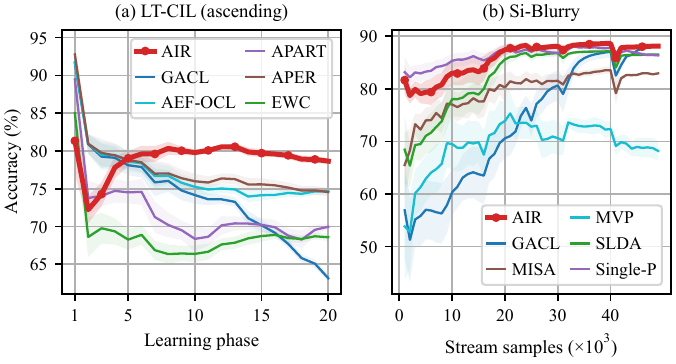}
        \caption{Six-seed learning-progress curves on ascending CIFAR-100 LT-CIL (left) and CIFAR-100 Si-Blurry (right). Lines show means and shaded regions show standard errors.}\label{fig:catastrophic_forgetting_progress}
    \end{figure}

\subsection{AIR Mitigates the Imbalance Issue}\label{sec:Exp_AIR_Imbalance}
\subsubsection{Reduced Headward Classification Bias}
    \par We inspect the final classifiers of GACL, DS-AL, and AIR on descending CIFAR-100 with $\rho=500$. As shown in the row-normalized confusion matrices in \Cref{fig:confusion_matrix}, GACL and DS-AL more frequently map tail-class samples to head classes, whereas AIR substantially reduces this headward bias.

    \begin{figure}[ht]
        \centering
        \includegraphics[width=\linewidth]{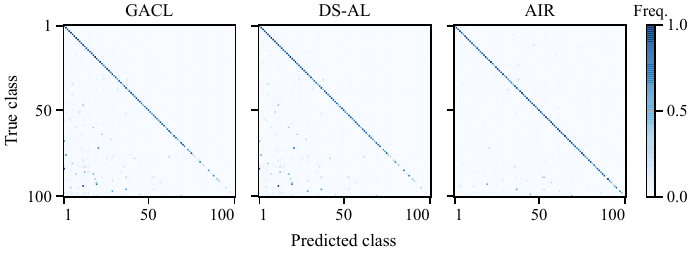}
        \caption{Row-normalized final confusion matrices for GACL, DS-AL, and AIR (left to right) on the seed-42 descending CIFAR-100 LT-CIL stream ($\rho=500$). GACL and DS-AL use $\gamma=50$ (with $\gamma_{\mathrm{comp}}=50$ for DS-AL), while AIR uses $\gamma=20{,}000$; all three use an 8,192-dimensional projection. Class indices run from head to tail.}\label{fig:confusion_matrix}
    \end{figure}

    \subsubsection{AIR Improves Tail Accuracy and Narrows the Head--Tail Gap}
        \par \Cref{fig:accuracy} sorts classes from rare to rich using their LT-CIL training counts and reports final class accuracy in 20 equal-rank bins (5 CIFAR-100 or 10 ImageNet-R classes per bin). Tail/middle/head summaries use the 7/6/7 bins whose centers lie in the lower/middle/upper rank thirds, respectively. Each panel contains CIFAR-100 and ImageNet-R results for GACL and AIR over six seeds. Averaged over the three orders, AIR changes tail/middle/head accuracy from 28.27\%/\allowbreak 71.74\%/\allowbreak 90.69\% to 65.20\%/\allowbreak 85.97\%/\allowbreak 85.82\% on CIFAR-100, reducing the head--tail gap from 62.42\% to 20.62\%. On ImageNet-R, the corresponding values change from 27.27\%/\allowbreak 60.60\%/\allowbreak 77.48\% to 56.18\%/\allowbreak 70.65\%/\allowbreak 73.10\%, reducing the gap from 50.21\% to 16.92\%. AIR thus substantially narrows the head--tail accuracy gap while retaining strong head-class accuracy, with head-accuracy reductions of 4.87\% and 4.38\% on CIFAR-100 and ImageNet-R, respectively.

        \begin{figure}[ht]
            \centering
            \includegraphics[width=\linewidth]{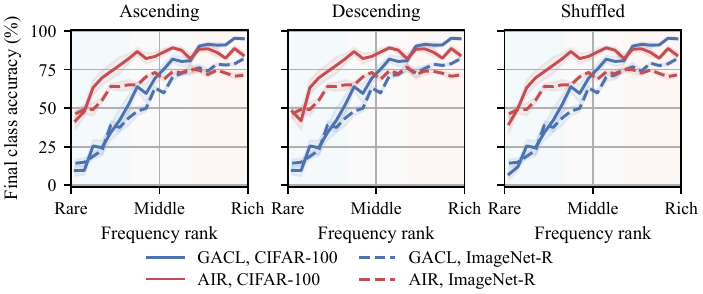}
            \caption{Final class accuracy from rare to rich training classes under ascending, descending, and shuffled LT-CIL. Each curve averages 20 equal-frequency-rank bins over six seeds; shading denotes standard error.}\label{fig:accuracy}
        \end{figure}

\subsection{AIR Alleviates Skewed Loss}\label{sec:Exp_Loss}
    \par We next measure the per-sample mean squared error (MSE) of each class on the test set. In the descending scenario, smaller class indices correspond to instance-rich head classes. In the seed-42 run shown in \Cref{fig:CIFAR100-MSE}, the mean head/tail MSEs are 0.0039/0.0112 for ACIL, 0.0035/0.0106 for DS-AL, and 0.0051/0.0067 for AIR. Thus, the tail-to-head MSE ratio decreases from 2.87 for ACIL and 3.01 for DS-AL to 1.32 for AIR, agreeing with the class-accuracy and confusion results in \Cref{sec:Exp_AIR_Imbalance}.

    \begin{figure}[ht]
        \centering
        \includegraphics[width=\linewidth]{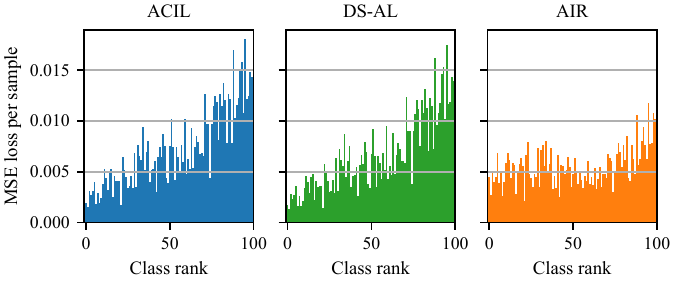}
        \caption{Per-sample test MSE by class for the seed-42 descending CIFAR-100 LT-CIL run ($\rho=500$). Class indices run from head to tail; AIR yields a substantially flatter class-wise loss profile.}\label{fig:CIFAR100-MSE}
    \end{figure}

\subsection{Comparing AIR and Existing ACL Methods}\label{sec:Exp_Comparision}
    \par Unweighted analytic ridge classifiers favor head-class mean losses through frequency-proportional weighting. AIR addresses this imbalance with ARM, which gives class-mean losses equal weight and can be integrated into existing analytic classifiers.

    \subsubsection{ARM as a Plug-in for Analytic Ridge Classifiers}
        \par ARM can replace the unweighted ridge head of analytic CL methods without changing their frozen feature extractors. Within every comparison in \Cref{tab:ARM_extension}, the baseline and ARM variant use the same stream, random projection of dimension $f=8192$. ACIL/AIR use a common $\gamma=200$ for all three orders; DS-AL/DS-AIR use $\gamma=\gamma_{\mathrm{comp}}=50$ throughout. Integrating ARM improves both average and last-phase accuracy for every arrival order and every one of the six paired seeds on long-tailed CIFAR-100 with $\rho=500$.

        \begin{table}[ht]
            \centering
            \scriptsize
            \def\shuf{Shuffled}
            \def\asen{Ascending}
            \def\dese{Descending}
            \newcommand{\armmean}[2]{\mbox{#1{\tiny$\pm$#2}}}
            \newcommand{\armimprovement}[4]{\armmean{#1}{#2}\par{\tiny(+#3$\pm$#4)}}

            \caption{Accuracy (\%) after adding ARM to ACIL and DS-AL on CIFAR-100 LT-CIL ($\rho=500$). Leading values are six-seed means$\pm$standard errors; parentheses on the ARM rows report the paired improvement as mean$\pm$standard error over the same seeds.}\label{tab:ARM_extension}
            \begin{tblr}[expand=ACIL w/o ARM & \armmean{73.28}{0.71} & \armmean{62.00}{0.44} & \armmean{83.99}{0.20} & \armmean{62.00}{0.44} & \armmean{69.66}{1.45} & \armmean{62.00}{0.44} \\
ACIL w/ ARM & \armimprovement{77.21}{0.74}{3.93}{0.09} & \armimprovement{72.60}{0.57}{10.60}{0.31} & \armimprovement{87.93}{0.14}{3.94}{0.13} & \armimprovement{72.60}{0.57}{10.60}{0.31} & \armimprovement{77.13}{1.64}{7.47}{0.23} & \armimprovement{72.60}{0.57}{10.60}{0.31} \\
\midrule
DS-AL w/o ARM & \armmean{73.90}{0.75} & \armmean{62.66}{0.46} & \armmean{84.70}{0.21} & \armmean{64.41}{0.52} & \armmean{70.97}{1.51} & \armmean{63.63}{0.56} \\
DS-AL w/ ARM & \armimprovement{75.61}{0.77}{1.71}{0.04} & \armimprovement{69.32}{0.57}{6.66}{0.26} & \armimprovement{86.73}{0.16}{2.03}{0.09} & \armimprovement{69.88}{0.56}{5.47}{0.21} & \armimprovement{74.58}{1.65}{3.61}{0.19} & \armimprovement{69.60}{0.59}{5.97}{0.20} \\
]{
                colspec={Q[l, m]*{6}{X[c, m]}}, colsep=1pt, cell{1}{1}={r=2,c=1}{c,m}, cell{1}{even[2]}={r=1,c=2}{c,m},
            }
                \toprule
                Method         &  \asen        &                &   \dese       &                &   \shuf        &                 \\ \cmidrule[lr]{2-3}\cmidrule[lr]{4-5}\cmidrule[l]{6-7}
                               &  \aavg        &   \alst        &   \aavg       &   \alst        &   \aavg        &   \alst         \\ \midrule
                ACIL w/o ARM & \armmean{73.28}{0.71} & \armmean{62.00}{0.44} & \armmean{83.99}{0.20} & \armmean{62.00}{0.44} & \armmean{69.66}{1.45} & \armmean{62.00}{0.44} \\
ACIL w/ ARM & \armimprovement{77.21}{0.74}{3.93}{0.09} & \armimprovement{72.60}{0.57}{10.60}{0.31} & \armimprovement{87.93}{0.14}{3.94}{0.13} & \armimprovement{72.60}{0.57}{10.60}{0.31} & \armimprovement{77.13}{1.64}{7.47}{0.23} & \armimprovement{72.60}{0.57}{10.60}{0.31} \\
\midrule
DS-AL w/o ARM & \armmean{73.90}{0.75} & \armmean{62.66}{0.46} & \armmean{84.70}{0.21} & \armmean{64.41}{0.52} & \armmean{70.97}{1.51} & \armmean{63.63}{0.56} \\
DS-AL w/ ARM & \armimprovement{75.61}{0.77}{1.71}{0.04} & \armimprovement{69.32}{0.57}{6.66}{0.26} & \armimprovement{86.73}{0.16}{2.03}{0.09} & \armimprovement{69.88}{0.56}{5.47}{0.21} & \armimprovement{74.58}{1.65}{3.61}{0.19} & \armimprovement{69.60}{0.59}{5.97}{0.20} \\

                \bottomrule
            \end{tblr}
        \end{table}

    \subsubsection{AIR Excels in Imbalanced Scenarios}
        \par We vary $\rho$ on descending CIFAR-100 in \Cref{fig:imb_factor}. The left panel compares GACL and AIR using exactly the same $\gamma=20{,}000$, 8,192-dimensional projection, frozen backbone, data order, and single-pass online stream. At $\rho=1$, every class count is equal and ARM assigns unit sample weights, so AIR and GACL coincide. This equality is also verified for every one of the six seeds. At $\rho=500$, GACL obtains 73.58\%/43.43\% average/final accuracy, whereas AIR obtains 89.09\%/78.65\%. The right panel uses $\gamma=\gamma_{\mathrm{comp}}=50$ for both DS-AL and DS-AIR, matching \Cref{tab:ARM_extension}. Adding ARM improves these metrics from 84.70\%/64.41\% to 86.73\%/69.88\% at $\rho=500$.

        \begin{figure}[ht]
            \centering
            \includegraphics[width=\linewidth]{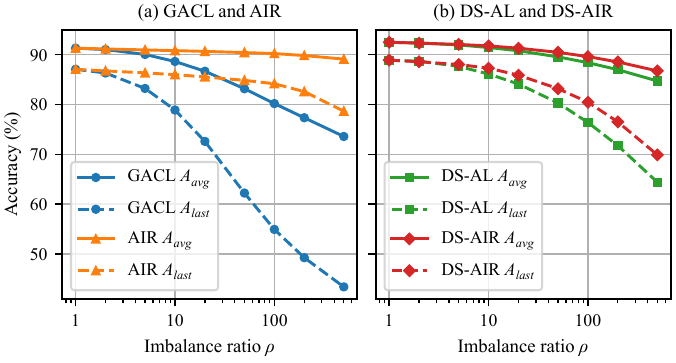}
            \caption{Six-seed mean average and last-phase accuracy on descending CIFAR-100 as $\rho$ varies. Both panels use matched hyper-parameters within each pair; GACL and AIR coincide at $\rho=1$.}\label{fig:imb_factor}
        \end{figure}

\section{Limitation: Matrix Storage in AIR}
    \par Storage here covers the retained matrices $\Mat{A}$, $\Mat{C}$, and $\Mat{W}_{k-1}$, not total memory usage. In the pseudocode, $\Mat{A}$ denotes AIR's $\Mat{Q}_{1:k}$ or G-AIR's per-class $\{\Mat{H}^{(y)}\}$, and $\Mat{C}$ denotes $\Mat{S}$. AIR retains one class-normalized Gram sum; G-AIR retains per-class Grams to reweight recurring classes. Their storage costs are $\Theta(f^2+fC_k)$ and $\Theta(C_kf^2+fC_k)$, respectively (\Cref{tab:memory-analysis}), excluding the backbone, activations, and temporary buffers such as AIR's phase-local Grams.

    \begin{table}[ht]
        \caption{Storage costs of retained matrices $\Mat{A}$, $\Mat{C}$, and $\Mat{W}_{k-1}$.}\label{tab:memory-analysis}
        \centering
        \footnotesize
        \begin{tblr}{colspec={Q[c,m]*{3}{X[c,m]}Q[c,m]}, colsep=1pt}
          \toprule
          Method       & $\Mat{A}$         & $\Mat{C}$      & $\Mat{W}_{k-1}$ & Total                \\ \midrule
          ACIL \& GACL & $\Theta(f^2)$     &       -        & $\Theta(fC_k)$  & $\Theta(f^2 + fC_k)$ \\
          RanPAC       & $\Theta(f^2)$     & $\Theta(fC_k)$ & $\Theta(fC_k)$  & $\Theta(f^2 + fC_k)$ \\ \midrule
          AIR          & $\Theta(f^2)$     & $\Theta(fC_k)$ & $\Theta(fC_k)$  & $\Theta(f^2 + fC_k)$ \\
          G-AIR        & $\Theta(f^2C_k)$  & $\Theta(fC_k)$ & $\Theta(fC_k)$  & $\Theta(f^2C_k+fC_k)$ \\
          \bottomrule
        \end{tblr}
    \end{table}
    \par We measure G-AIR's GPU and host memory on ImageNet-R Si-Blurry using a 24-GiB RTX 4090, with class statistics in system RAM and feature extraction and classifier updates on the GPU. We vary $f\in\{512,1024,2048,4096,8192\}$ over six data seeds (42--47). Memory measurements use $\gamma=2000$, three updates per batch, and phase-end evaluation; the accompanying accuracy curves use the single-pass protocol with $\gamma=1000$ from \Cref{sec:Exp_GCIL}.

    \begin{figure}[ht]
        \centering
        \includegraphics[width=\linewidth]{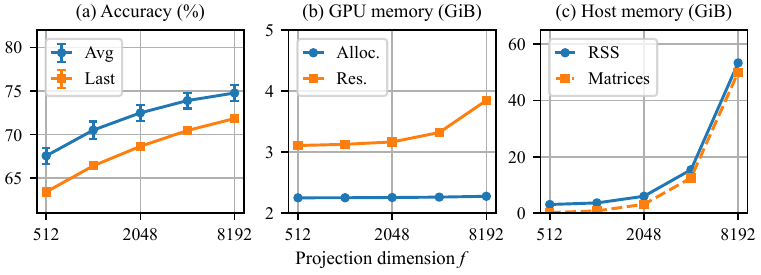}
        \caption{G-AIR on ImageNet-R with 32-bit float statistics: (a) phase-average/final accuracy (mean$\pm$standard error); (b) peak allocated/reserved GPU memory; (c) peak whole-process host memory (RSS) and theoretical per-class Gram storage. Measured memory curves show the maximum of six per-run peaks; the dashed storage curve excludes $\Mat{C}$ and $\Mat{W}_{k-1}$.}\label{fig:gair_memory_tradeoff}
    \end{figure}

    \par Across these widths (\Cref{fig:gair_memory_tradeoff}), per-class Gram storage grows from 0.195 to 50~GiB and peak process RSS from 3.050 to 53.313~GiB, while peak allocated/reserved GPU memory remains below 2.28/3.85~GiB. Average/final accuracy increases from 67.57\%/63.43\% to 74.78\%/71.85\%. CPU offloading limits GPU-memory demand, and reducing $f$ controls the quadratic storage cost through an accuracy--storage trade-off, supporting deployment on lower-cost, memory-constrained devices.

\section{Conclusions}
    \par We reveal that uniform sample weighting assigns larger coefficients to head-class mean losses, motivating AIR, an online exemplar-free method for class-imbalanced CL. Its ARM module uses normalized inverse-frequency weighting to equalize total sample weights across classes. Sufficient-statistic updates yield a closed-form incremental classifier that provably recovers the joint solution of the same weighted ridge objective for fixed features, with G-AIR extending this capability to recurring-class streams.
    \par Across six data seeds, AIR improves aggregate accuracy and macro F1 over the strongest exemplar-free LT-CIL baselines by 3.21\% and 2.14\%, respectively. G-AIR leads 15 Si-Blurry baselines, with corresponding gains of 2.32\% and 1.27\% over the strongest baseline. One-sided paired $t$-tests support positive mean gains for all four comparisons (Holm-adjusted $p\leq0.0054$). ARM also improves ACIL and DS-AL under matched settings, demonstrating its broader applicability to analytic CL.
    \par G-AIR may require substantial memory when there is too many classes or extremely high-dimensional projections. Future work may reduce this cost through low-rank factorization. In this paper, we use only normalized inverse-frequency weights $\pi$ for simplicity. Exploring alternative or user-specific weights in the future may improve performance or achieve personalized continual learning.

\section*{Data Availability}
All datasets used in this study are publicly available. CIFAR-100 can be
downloaded from
\url{https://cave.cs.toronto.edu/kriz/cifar.html}, ImageNet-R from
\url{https://people.eecs.berkeley.edu/~hendrycks/imagenet-r.tar}, CUB-200-2011
from
\url{https://www.vision.caltech.edu/datasets/cub_200_2011/},
and CORe50 (128x128) from
\url{https://vlomonaco.github.io/core50/}.
The open ViT-B/16 weights, pre-trained on ImageNet-21K and fine-tuned
on ImageNet-1K, are available from
\url{https://huggingface.co/timm/vit_base_patch16_224.augreg2_in21k_ft_in1k}.
The source code will be made publicly available upon acceptance at
\url{https://github.com/fang-d/AIR}.

\section*{Acknowledgements}
This work is partially supported by High Performance Computing Platform of South China University of Technology.

\printbibliography
\appendix
\crefalias{section}{appendix}
\numberwithin{equation}{section}
\numberwithin{figure}{section}
\numberwithin{table}{section}

\section{Verification of the Weight-Invariant Property}\label{sec:Exp_WeightInvariance}
    \par We verify \Cref{thm:recursive} on CIFAR-100 under the LT-CIL and Si-Blurry settings in \Cref{sec:Exp_LT-CIL,sec:Exp_GCIL}, using a fixed 512-dimensional ReLU random projection, $\gamma=100$, and data seeds 42 and 43, without hyper-parameter tuning. Each phase-end comparison matches AIR's incremental weights against an independent joint weighted-ridge solution recomputed from all observed samples. Across all 130 comparisons, the relative Frobenius error $\epsilon_W=\FN{\Mat{W}_{\mathrm{CL}}-\Mat{W}_{\mathrm{joint}}}/\FN{\Mat{W}_{\mathrm{joint}}}$ never exceeds $6.04\times10^{-14}$, with identical test predictions (\Cref{tab:weight_invariance}). An unweighted joint reference differs at an imbalanced prefix in every stream. These results support joint--continual equivalence to floating-point precision for fixed features and matched weighted objectives.

    \begin{table}[ht]
        \centering
        \caption{Joint--continual equivalence on cached CIFAR-100 features. Comparison counts equal phases $\times$ two seeds; errors are maxima across comparisons. Final test accuracies average the two seeds for continual (CL) and joint learning.}\label{tab:weight_invariance}
        \small
        \begin{tblr}{
            colspec={X[2.3,l,m]X[1.5,c,m]X[1.5,c,m]X[c,m]X[c,m]},
            colsep=1pt,hline{2}={4-5}{0.4pt},
            cell{1}{1,2,3}={r=2}{c,m},cell{1}{4}={c=2}{c,m}
        }
            \toprule
            Setting & {Phase-end\\comparisons} & $\max\epsilon_W$ & Final accuracy (\%) & \\
                    &        &                   & CL & Joint \\
            \midrule
            LT-CIL ascending  & 40 & $5.02\times10^{-14}$ & 76.11 & 76.11 \\
            LT-CIL descending & 40 & $4.58\times10^{-14}$ & 76.11 & 76.11 \\
            LT-CIL shuffled   & 40 & $6.03\times10^{-14}$ & 76.11 & 76.11 \\
            Si-Blurry         & 10 & $5.81\times10^{-14}$ & 84.31 & 84.31 \\
            \bottomrule
        \end{tblr}
    \end{table}

\section{Formal Definition of the Evaluation Metrics}\label{sec:metrics}
    \par Let $a_k$ be the test accuracy over all classes observed through phase $k$, evaluated at the end of that phase. If the phase-$k$ evaluation contains $M_k$ samples, then
    \begin{equation}
        a_k=\frac{1}{M_k}\sum_{i=1}^{M_k}\mathbf{1}[\hat{y}_{k,i}=y_{k,i}].
    \end{equation}
    For a $K$-phase CL task, average accuracy (\aavg) is $\maavg = \frac{1}{K} \sum_{k=1}^{K} a_k$, while the last-phase accuracy (\alst) is $\malst = a_K$.

    \par For online streams, let $n_t$ denote the number of unique training samples observed at periodic evaluation $t$, and let $T$ be the number of periodic evaluations. The reported discrete periodic-evaluation mean (\aauc, conventionally labeled AUC) is
    \begin{equation}
        \maauc = \frac{1}{T}\sum_{t=1}^{T}\mathcal{A}(n_t),
    \end{equation}
    where $\mathcal{A}(n_t)$ is the accuracy on the classes exposed after $n_t$ unique stream samples. Starting at threshold 1,000, evaluation occurs after the first completed stream batch whose cumulative unique-sample count strictly exceeds the threshold; the threshold then increases by 1,000 without resetting at phase boundaries. Thus $n_t>1000t$ rather than $n_t=1000t$, and a batch ending exactly at a threshold does not trigger evaluation. Batch repeats do not increase this count. No initial or additional terminal evaluation is inserted: the final batch contributes only if it crosses the next threshold. Any remaining tail receives no separate term. AUC is neither a trapezoidal integral nor weighted by sample intervals; phase-end Avg and Last are computed separately. The same rules apply to macro F1 AUC.

    \par Let $\operatorname{TP}_{k,y}$, $\operatorname{FP}_{k,y}$, and $\operatorname{FN}_{k,y}$ be the phase-$k$ counts for class $y$. The phase-$k$ macro F1 is
    \begin{equation}
        f_k=\frac{1}{C_k}\sum_{y=0}^{C_k-1}
        \frac{2\operatorname{TP}_{k,y}}
        {2\operatorname{TP}_{k,y}+\operatorname{FP}_{k,y}+\operatorname{FN}_{k,y}}.
    \end{equation}
    \noindent Average macro F1 (\favg), last-phase macro F1 (\flst), and macro F1 AUC (\fauc) are
    \begin{equation}
        \mfavg = \frac{1}{K} \sum_{k=1}^{K} f_k, \qquad
        \mflst = f_K, \qquad
        \mfauc = \frac{1}{T}\sum_{t=1}^{T}\mathcal{F}(n_t),
    \end{equation}
    \noindent where $\mathcal{F}(n_t)$ is macro F1 at periodic evaluation $t$. Each main-comparison table cell reports the mean over six data seeds and its standard error, $s/\sqrt{6}$, where $s$ is the sample standard deviation across seeds.

\end{document}